\documentclass[twoside]{article}

\usepackage[accepted]{aistats2026}

\usepackage{amsthm, amssymb, bbm, amsmath, bm}
\usepackage{xcolor}
\usepackage{fontawesome5}
\usepackage{graphicx}
\usepackage[hidelinks, backref=page]{hyperref}
\usepackage[english]{babel}
\usepackage{cleveref}
\usepackage{dirtytalk}
\usepackage{thmtools}
\usepackage{xfrac}
\usepackage{lmodern}
\theoremstyle{plain}

\theoremstyle{definition}

\usepackage[round]{natbib}

\renewcommand*{\backref}[1]{}
\renewcommand*{\backrefalt}[4]{
    \ifcase #1 Not cited.
    \or        Cited on page~#2.
    \else      Cited on pages~#2.
    \fi
}

\begin{document}

\runningauthor{Doumont, Fan, Maus, Gardner, Moss, Pleiss}

\twocolumn[
  \aistatstitle{We Still Don’t Understand High-Dimensional Bayesian Optimization}

  \aistatsauthor{
    Colin Doumont\textsuperscript{$*$} \And Donney Fan \And Natalie Maus }
  \aistatsaddress{
      Tübingen AI Center \\ \href{mailto:c.doumont@uni-tuebingen.de}{\texttt{c.doumont@uni-tuebingen.de}} \And University of British Columbia \\ Vector Institute \And University of Pennsylvania }
  \aistatsauthor{
      Jacob R. Gardner \And Henry Moss \And Geoff Pleiss }
  \aistatsaddress{
      University of Pennsylvania \And Lancaster University \\ University of Cambridge \And University of British Columbia \\ Vector Institute }
  \begin{center}
    \vspace{-0.075in} \normalsize \href{https://github.com/colmont/linear-bo}{\faGithub\ \textcolor{blue}{\texttt{https://github.com/colmont/linear-bo}} } \vspace{0.3in}
  \end{center}
]

\begin{abstract}
    Existing high-dimensional Bayesian optimization (BO) methods aim to overcome the curse of dimensionality by carefully encoding structural assumptions, from locality to sparsity to smoothness, into the optimization procedure. Surprisingly, we demonstrate that these approaches are outperformed by arguably the simplest method imaginable: Bayesian linear regression. After applying a geometric transformation to avoid boundary-seeking behavior, Gaussian processes with linear kernels match state-of-the-art performance on tasks with 60- to 6,000-dimensional search spaces. Linear models offer numerous advantages over their non-parametric counterparts: they afford closed-form sampling and their computation scales linearly with data, a fact we exploit on molecular optimization tasks with $>~$20,000 observations. Coupled with empirical analyses, our results suggest the need to depart from past intuitions about BO methods in high-dimensions.
\end{abstract}

\section{\bfseries\small INTRODUCTION}

High-dimensional search spaces have historically been a challenge for Bayesian optimization (BO). Classic theoretical results demonstrate that regret increases exponentially with dimensionality \citep{srinivas2010gaussian,bull2011convergence} and na\"ive BO implementations can even perform worse than random search on problems exceeding 10 dimensions \citep{wang2016bayesian, santoni2024comparison}. The most successful high-dimensional BO algorithms have exploited or imposed structural assumptions on the problem's objective function, such as \textit{additive decompositions} \citep{kandasamy2015high,gardner2017discovering,mutny2018efficient,wang2018batched}, \textit{locality} \citep{eriksson2019scalable,muller2021local,wu2023behavior}, \textit{geometric warpings} \citep{oh2018bock,kirschner2019adaptive}, or \textit{sparsity} \citep{eriksson2021high,papenmeier2022increasing}. These approaches achieve better performance than na\"ive methods and even offer asymptotic reductions in regret under appropriate structural conditions \citep{kandasamy2015high,wu2023behavior}. Nevertheless, these algorithms are often complicated, introduce many hyperparameters, and may require many observations for meaningful optimization.

\renewcommand{\thefootnote}{\fnsymbol{footnote}}
\footnotetext[1]{Work done while at the Vector Institute.}
\renewcommand{\thefootnote}{\arabic{footnote}}

Recently, \citet{hvarfner2024vanilla} and \citet{xu2025standard} achieved breakthrough results on high-dimensional problems without any structural assumptions, relying only on Gaussian process priors that favor simple functions as dimensionality $D$ increases. Specifically, after scaling kernel lengthscale hyperparameters proportionally to $\sqrt{D}$ to encourage smoothness, they show that the decades-old ``Vanilla'' BO recipe---expected improvement using a Gaussian process with a squared-exponential kernel \citep{movckus1974bayesian,jones1998efficient}---achieves competitive performance across problems up to $D>6000$. Perhaps most remarkably, they demonstrate empirical success in a seemingly intractable regime, where the observation budget $N$ is on the same order as the search space dimensionality $D$ (which we denote by $N \approx D$). With $N \approx D$, one cannot even build a first-order Taylor approximation around a single point (i.e.\ a locally linear model), and yet Vanilla BO outperforms methods specifically engineered for high-dimensional problems.

Recognizing that true global optimization is all but impossible when $N \approx D$, we push this smoothness approach to its logical extreme. Specifically, we restrict the surrogate model to only have support for \emph{linear functions}, the smoothest functions under most mathematical definitions. With linear models, we can globally model linear functions or locally model non-linear functions, but not both. Nevertheless, we hypothesize that the simplicity of these models may be beneficial, particularly in the $N \approx D$ regime. Specifically, if observation budgets cannot support learning beyond locally-linear approximations, we hypothesize that simple Bayesian linear surrogate models are viable alternatives to non-parametric Gaussian processes. Linear models could also prove valuable in the $N \gg D$ setting: their $O(D)$ parametric representation affords closed-form sampling and $O(ND^2)$ computational complexity, whereas their non-degenerate Gaussian process counterparts require $O(N^3)$ computation and do not admit exact pathwise samples.

Standard BO references do not recommend the use of linear models, or equivalently Gaussian processes with linear kernels \citep{frazier2018bayesian,garnett2023bayesian}. Indeed, our own experiments show that linear kernels in their most na\"ive form perform extremely poorly on most high-dimensional BO tasks (see \Cref{fig:main-plot}). However, our analyses reveal that this poor performance is largely attributable to boundary-seeking behavior of linear models, which we address through simple geometric modifications to the kernel. Our most significant proposal is to bijectively map the input space (typically a hypercube in $\mathbb R^D$) onto a hypersphere in $\mathbb R^{D+1}$, yielding provable immunity from boundary-seeking behavior and resulting in a kernel that is a function of cosine similarity between pairs of inputs.

Surprisingly, while standard linear kernels fail to yield meaningful optimization, our spherically-projected linear kernels match or exceed state-of-the-art performance across benchmarks from $D=60$ to $D > 6000$. Most notably, our sphere-domain linear models excel when $N \gg D$ and exact Gaussian process inference becomes computationally prohibitive, providing a significant performance improvement over existing scalable approaches like stochastic variational Gaussian processes \citep{hensman2013gaussian,vakili2021scalable,moss2023inducing}. Beyond demonstrating the state-of-the-art performance of our spherically-projected linear kernels, especially in large-$N$ settings, these results further question our understanding of high-dimensional Bayesian optimization. The fact that the simplest possible surrogate model (after correcting for geometric pathologies) can match or outperform all other methods suggests the need for a radical departure from the conventional wisdom and theory in the field.

\section{\bfseries\small BACKGROUND \& RELATED WORK}

\subsection{Bayesian Optimization} \label{sec:BO}
The goal of Bayesian optimization \citep{garnett2023bayesian} is to find the global optimum $\mathbf{x}^* \in \mathcal{X}$ of a black-box function $f: \mathcal{X} \rightarrow \mathbb{R}$, where $f$ can only be evaluated point-wise and those evaluations are considered \say{expensive} (e.g.\ in time or money). Depending on the setting, these evaluations of $f$ can also be noisy, i.e.\ $f(\mathbf{x}) + \varepsilon_i$ with $\varepsilon_i \sim \mathcal{N}\left(0, \sigma^2_{\varepsilon}\right)$. Although $\mathcal{X}$ can represent different spaces, the typical space in BO is the unit hypercube, i.e.\ $\mathcal{X} = [0,1]^D$. In this paper, we consider the \emph{centered} hypercube $[-1, 1]^D$, which can easily be mapped to or from $[0,1]^D$.

To find $\mathbf{x^*}$, BO iteratively selects new points at which to evaluate $f$, where these points are chosen by maximizing an acquisition function $\alpha_t(\cdot)$, i.e.\ $\mathbf{x}_{t+1} = \arg \max_{\mathbf{x} \in \mathcal{X}} \alpha_t(\mathbf{x})$, which balances exploration and exploitation by relying on a (probabilistic) surrogate model of $f$, for example a Gaussian process (GP).

\paragraph{Gaussian Processes.}
GPs define distributions over functions that are specified by a mean function \(\mu(\mathbf{x})\) (typically a constant) and a covariance function or kernel \(k(\mathbf{x}, \mathbf{x'}) : \mathcal{X} \times \mathcal{X} \rightarrow \mathbb{R}\). GPs can be equivalently viewed as Bayesian linear regression with a (potentially infinite-dimensional) basis expansion, where \(k(\mathbf{x}, \mathbf{x'})\) corresponds to the inner product under this featurization. It is common to use kernels that correspond to infinite-dimensional basis expansions, such as the Radial Basis Function (RBF):
\begin{gather*}
    k_\textrm{RBF}(\mathbf{x}, \mathbf{x}') := \exp\left(-\tfrac{1}{2} \left\lVert \mathbf{z} - \mathbf{z'} \right\rVert^2\right), \\
    \mathbf{z} := \begin{bmatrix} \tfrac{x_1}{\ell_1} & \cdots  & \tfrac{x_D}{\ell_D} \end{bmatrix}, \quad
    \mathbf{z'} := \begin{bmatrix} \tfrac{x'_1}{\ell_1} & \cdots  & \tfrac{x'_D}{\ell_D} \end{bmatrix}.
\end{gather*}
Here, the lengthscales \(\ell_i > 0\) are hyperparameters that re-scale the search space \(\mathcal{X}\) before applying the kernel. Recent work has shown the necessity of scaling these lengthscales with \(\sqrt{D}\), so that the expected norms of the re-scaled inputs $\mathbf z$, $\mathbf z'$ are independent of dimensionality \citep{hvarfner2024vanilla, xu2025standard, papenmeier2025understanding}.

\paragraph{Degenerate Kernels and Linear Models.} While kernels corresponding to infinite basis expansions (e.g.\ the RBF kernel) often yield universally approximating functions, kernels corresponding to finite basis expansions have limited representational capacity.
In particular, the \emph{linear kernel}, defined as
\begin{equation}
     k_\mathrm{linear}(\mathbf{x}, \mathbf{x}') := b_0 + b_1 \mathbf{x}^\top \mathbf{x'}
     \label{eqn:std_linear_kernel}
\end{equation}
for some hyperparameters $b_0, b_1 > 0$, only corresponds to the ($D+1$)-dimensional basis expansion $[\sqrt{b_0}, \sqrt{b_1} x_1, \ldots, \sqrt{b_1} x_D]$. Similarly, polynomial kernels of the form $\left( b_0 + b_1 \mathbf{x}^\top \mathbf{x'} \right)^m$ with $m \in \mathbb N$ correspond to $O(D^m)$-dimensional basis expansions. Usually, these kernels are avoided for BO in favor of their universal approximating counterparts \citep{garnett2023bayesian}.

\subsection{High-Dimensional Bayesian Optimization} \label{sec:high-dim-BO}
Classic results show that the regret of BO with universally approximating kernels increases exponentially with $D$ \citep{srinivas2010gaussian, bull2011convergence}, and thus problems with $\geq 20$ dimensions can prove challenging for BO \citep{frazier2018bayesian}. To circumvent this curse of dimensionality, a wide range of methods have been designed specifically for high-dimensional Bayesian optimization (HDBO), often relying on complex algorithms and strong assumptions on the objective function. We keep a more detailed literature review of these methods for \Cref{app:hdbo}, but present the main high-level ideas in this paragraph: common approaches either impose structural assumptions on the objective function, such as an additive decomposition over the input features \citep[e.g.][]{kandasamy2015high} or sparsity over projected features \citep[e.g.][]{papenmeier2023bounce}, or reconfigure the BO algorithm to perform local rather than global optimization \citep[e.g.][]{eriksson2019scalable}. All of these algorithms can be seen as approaches towards reducing the complexity of the objective function or search space.

\paragraph{Dimensionality-Based Smoothness.} A recent line of work by \citet{hvarfner2024vanilla} demonstrates that neither structural assumptions nor explicitly local algorithms are necessary for successful HDBO.
Instead, they propose increasing the kernel lengthscale hyperparameters at a rate of $\sqrt D$---equivalent to re-scaling the $\mathcal X$ hypercube by $1/\sqrt{D}$---to yield an objective function prior that favors smoothness as $D$ increases. This \say{Vanilla BO} recipe, alongside contemporaneous methods with similar lengthscale scalings \citep{xu2025standard,papenmeier2025understanding}, significantly outperforms the methods described above on benchmarks with up to $D=6000$ dimensions and as few as $N = 1000$ observations.

\paragraph{Geometric Transformations.} Few works propose geometric warpings of the inputs beyond element-wise scaling through lengthscales.
A notable exception is \texttt{BOCK} \citep{oh2018bock}, who, like us, bijectively map the inputs into a different space before applying a kernel. We note several differences between this method and ours. Superficially, our proposed method maps inputs to a hypersphere, whereas \texttt{BOCK} maps to a hypercylinder. More substantially, our geometric warping is motivated to obtain meaningful performance from linear kernels, whereas \texttt{BOCK} applies this transformation to kernels with high representational capacity.

\section{\bfseries\small METHOD} \label{sec:method}
The recent works of \citet{hvarfner2024vanilla} and \citet{xu2025standard} have demonstrated the necessity of scaling lengthscales with dimensionality, effectively creating a prior that favors increasing smoothness as dimensionality increases. In this work, we investigate if further improvements are possible if we take this smoothness idea to its logical extreme. Importantly, we cannot simply consider a faster rate of lengthscale scaling, as this would lead to $\Vert \mathbf{x}/\ell - \mathbf{x'}/\ell \Vert \to 0$ as $D \to \infty$, effectively leading to a prior over constant functions. Instead, we turn to what is (under most definitions) the smoothest model of non-constant functions: Gaussian processes with linear kernels.

The idea of using linear kernels (i.e.\ Bayesian linear regression) for HDBO may appear surprising, especially given that standard BO references discourage its use \citep{shahriari2015taking, garnett2023bayesian}: the black-box objective functions modeled during BO are often of unknown complexity, necessitating non-parametric and universal approximating surrogates rather than simple parametric models. However, we argue that the limited fidelity of linear models is not a limitation in $N \approx D$ settings, where there are barely enough observations to ``fill up'' the capacity of linear models. Instead, we hypothesize that linear kernels can be competitive for HDBO tasks after accounting for pathologies in the optimization geometry.

Below, we propose a modified linear kernel for Bayesian optimization. Like the standard linear kernel, our proposed kernel is rank $O(D)$ and yields an extremely simple prior over functions. However, as we will demonstrate, our kernel significantly surpasses the BO performance of the standard linear kernel and even rivals the most sophisticated methods.

\begin{figure}[t!]
    \centering
    \includegraphics[width=0.83\columnwidth]{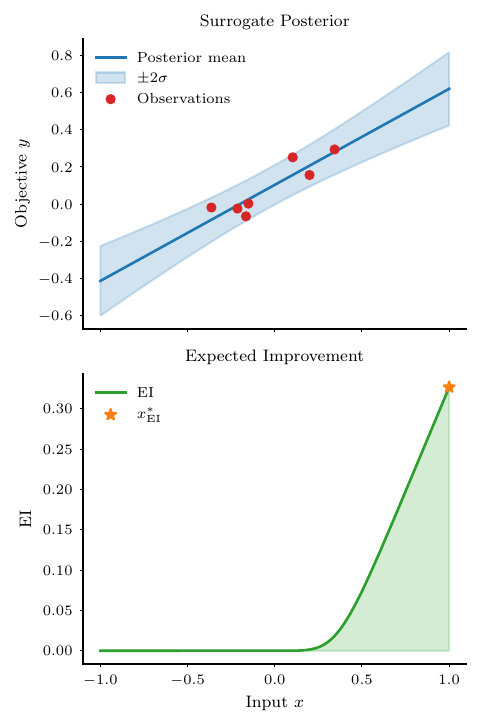}
    \caption{\textbf{Boundary-seeking behavior of linear models (\Cref{thm:boundary}).} \textit{(Top)} A Bayesian linear model's posterior mean and variance strictly increase away from the origin, i.e.\ $x=0$. \textit{(Bottom)} This causes most well-known acquisition functions (e.g.\ expected improvement) to be maximized exactly at the boundaries, discarding most of the search space. We resolve this pathology via a spherical projection of the inputs, allowing simple linear models to achieve state-of-the-art performance on high-dimensional problems.}
    \label{fig:1d_viz}
\end{figure}

\subsection{Proposed Linear Model} \label{sec:prop-lin-model}
We propose the following linear kernel for GPs (equivalent to Bayesian linear regression):
\begin{equation} \label{eq:lin-kernel}
    k_\textrm{linear}\left(\mathbf{x}, \mathbf{x'}\right) := b_0 + b_1 \; P\left(\mathbf{z}\right)^{\top} P\left(\mathbf{z'}\right),
\end{equation}
where $b_i > 0$ are hyperparameters, $P: \mathbb{R}^D \rightarrow \mathbb{S}^D$ is a bijective mapping onto the unit sphere, and $\mathbf{z},\mathbf{z'}$ are scaled versions of $\mathbf{x}, \mathbf{x'}$, i.e.\
\begin{equation} \label{eq:glob-l}
    \mathbf{z} := \tfrac{1}{a} \begin{bmatrix} \tfrac{x_1}{\ell_1} & \cdots  & \tfrac{x_D}{\ell_D} \end{bmatrix}, \; \; \mathbf{z'} := \tfrac{1}{a} \begin{bmatrix} \tfrac{x'_1}{\ell_1} & \cdots  & \tfrac{x'_D}{\ell_D} \end{bmatrix}
\end{equation}
for hyperparameters $a \in \mathbb R^+$, $\boldsymbol \ell \in \mathbb R^D$. Of all the modifications, the spherical mapping is the most impactful (confirmed by ablations in \Cref{app:other-abl.}). We now define each of these changes more rigorously.

\paragraph{Spherical Mapping.}
Linear models induce a boundary-seeking behavior during optimization, which we hypothesize is the primary factor affecting their performance. Specifically, the posterior mean and variance of linear models increase away from the center of the hypercube, resulting in maximal acquisition values on the search space boundary (implicitly discarding most of $\mathcal X$). We formalize this idea in the following theorem (see \Cref{app:boundary} for a proof), and provide a visualization for $D = 1$ in \Cref{fig:1d_viz}.

\begin{restatable}{thm}{boundary} \label{thm:boundary}
    For acquisition functions increasing in posterior mean and variance (e.g.\ expected improvement), Bayesian linear models will maximize acquisition on the boundary of the search space. That is, for $\mathbf{x}_{t+1} = \arg \max_{\mathbf{x} \in [-1, 1]^D} \alpha_t(\mathbf{x})$ at any timestep $t$, we have that $\mathbf{x}_{t+1}$ contains at least one dimension equal to $-1$ or $1$ (i.e.\ $\|\mathbf{x}_{t+1}\|_\infty = 1$).
\end{restatable}

However, when points lie on the unit hypersphere $\mathbb{S}^{D} \subset \mathbb{R}^{D+1}$ instead of the hypercube $[-1, 1]^D \subset \mathbb{R}^D$, the acquisition function cannot increase by scaling inputs, since all points are equally far from the origin (i.e.\ on a sphere) and thus \Cref{thm:boundary} does not apply (see counterexample in \Cref{app:no-boundary}). Thus, we propose mapping our inputs $\mathbf x$ from $\mathcal X$ to a unit sphere (via $P$ in Equation~\ref{eq:lin-kernel}) before applying the linear kernel.

While numerous mappings from $\mathbb R^D \to \mathbb S^D$ exist, we choose $P$ to be the \emph{inverse stereographic projection}:
\begin{equation} \label{eq:inv-stereo}
    P(\mathbf{z}) := \tfrac{1}{||\mathbf{z}||^2 + 1} \begin{bmatrix}2 z_1 & \cdots  & 2 z_D & ||\mathbf{z}||^2 - 1\end{bmatrix}.
\end{equation}
As can be seen from the equation, the projection scales all dimensions of the original $D$-dimensional vector by their norm (with some constants), and only introduces one additional dimension, resulting in $D+1$ features for our linear model. This mapping possesses a number of attractive characteristics: it is smooth, bijective over $\mathbb R^D$, and conformal, meaning that it preserves angles at which curves meet \citep{hilbert2021geometry}. Most notably, it is a (near) identity mapping for unit norm $\mathbf z$ (i.e.\ $P(\mathbf z) = [\mathbf z, 0]$ when $\Vert \mathbf z \Vert = 1$), a fact we will explore in depth in \Cref{sec:analysis}.

\paragraph{Decoupled Magnitude and Direction for $\ell_i$.}
To ensure a meaningful distance is maintained between all points as the dimensionality $D$ grows, recent methods have scaled the lengthscales $\ell_i$ by $\sqrt{D}$, either through the hyperprior \citep{hvarfner2024vanilla,xu2025standard} or the initialization value \citep{papenmeier2025understanding}. As a simplified alternative to impose this scaling, we decouple the magnitude and direction of the lengthscale vector. Specifically, we introduce a \emph{global} lengthscale $a$ in \Cref{eq:glob-l}, initialized to be $O(\sqrt{D})$, which we multiply by an (unconstrained) vector $\boldsymbol \ell$ that models varying sensitivities of each dimension. In the context of our kernel, the $O(\sqrt D)$ scaling on $a$ ensures that the entries of $\mathbf z$ are $O(1/\sqrt{D})$, ensuring our scaled points have $O(1)$ norm before projecting them onto the unit sphere. We place a log-normal prior on the vector $\boldsymbol \ell$, namely $\ell_i \sim \mathcal{LN}\left(\sqrt{2}, \sqrt{3}\right)$, though we find that performance is largely unaffected by this choice (see \Cref{app:hyperpriors}).

\paragraph{No Implicit Outputscale.}
A useful property of RBF kernels is that they are bounded above, i.e.\ $k(\mathbf x, \mathbf x') \leq 1$, achieving the maximum if and only if $\mathbf x = \mathbf x'$. The latter property does not hold for standard linear kernels; however, it is naturally enforced for our spherically-projected inputs. Since $P$ bijectively maps to the unit hypersphere, we have $P(\mathbf{z})^\top P(\mathbf{z'}) \leq 1$ and thus $k_\mathrm{linear}(\mathbf x, \mathbf x') \leq b_0 + b_1$ with equality if and only if $\mathbf x = \mathbf x'$. Inspired by \citet{hvarfner2024vanilla}, we further enforce that $b_0 + b_1 = 1$, removing any implicit outputscale that could affect optimization. Since $b_0, b_1 > 0$, we enforce this constraint by learning unconstrained parameters, and then mapping them to the simplex through the softmax function.

\subsection{Polynomial Extension}
We can naturally extend the logic of our linear kernel to higher-order polynomial kernels via:
\begin{equation} \label{eq:poly-kernel}
    k_\textrm{poly}\left(\mathbf{x}, \mathbf{x'}\right) := {\textstyle \sum_{i=0}^m b_i \left[P\left(\mathbf{z}\right)^{\top} P\left(\mathbf{z'}\right) \right]^i},
\end{equation}
where $m \in \mathbb N$ indicates the polynomial order. We again constrain the coefficients $b_i$ to the simplex through a softmax function. Kernels of this form are trivially positive semi-definite through a standard application of kernel composition rules \citep[e.g.][]{genton2001classes}.

\Cref{eq:poly-kernel} can approximate many kernels on $P(\mathbf z), P(\mathbf z') \in \mathbb{S}^D$ to an arbitrary degree of precision. \citet{schoenberg1942positive} proves that any dot-product kernel on the unit sphere (i.e.\ where the kernel is a function of $P(\mathbf z), P(\mathbf z')$) takes the form of \Cref{eq:poly-kernel} for some value of $m$. Moreover, any kernel that is a function of Euclidean distance (e.g.\ the Mat\'ern or rational quadratic kernels) can be approximated by \Cref{eq:poly-kernel} for inputs $P\left(\mathbf{z}\right), P\left(\mathbf{z'}\right) \in \mathbb{S}^{D}$ (see \Cref{app:isotropic-dot-product}). For example, \Cref{app:taylor-exp} shows that the RBF kernel restricted to the unit sphere has the Taylor expansion
\begin{equation} \label{eq:taylor-exp}
    k_\textrm{RBF}\left(P\left(\mathbf{z}\right), P\left(\mathbf{z'}\right)\right) = \sum_{i=0}^{\infty} \frac{1}{i!\,e} \left[P\left(\mathbf{z}\right)^{\top} P\left(\mathbf{z'}\right)\right]^i.
\end{equation}
Therefore, increasing $m$, our proposed model can approximate the RBF kernel (on a hypersphere) with any arbitrary precision. Surprisingly, as will become clear in \Cref{sec:ablations}, there is rarely a need to go beyond $m=1$ to match state-of-the-art HDBO performance.

\subsection{Advantages of Linear Models}
Before considering optimization performance, it is worth remarking on the practical benefits afforded by our kernel (or any kernel corresponding to an $O(D)$ feature expansion), where these benefits cannot be obtained by standard (e.g.\ RBF or Mat\'ern) kernels.

\paragraph{Computational Scalability.}
Since GPs with linear kernels correspond to an $O(D)$ feature map, we can perform exact posterior inference over the $O(D)$-dimensional parameters in $O(ND^2)$ time \citep{williams2006gaussian}. In contrast, RBF and Mat\'ern kernels, which correspond to infinite-dimensional basis expansions, require $O(N^3)$ computation for exact posterior inference. As a result, linear kernels can be used to scale to optimization problems where $N \gg D$ (e.g.\ molecular optimization in \Cref{sec:n>>d}), whereas standard kernels require posterior approximations.

\begin{figure*}[t!]
    \centering
    \begin{minipage}[t]{0.74\textwidth}
        \centering
        \includegraphics[width=\textwidth]{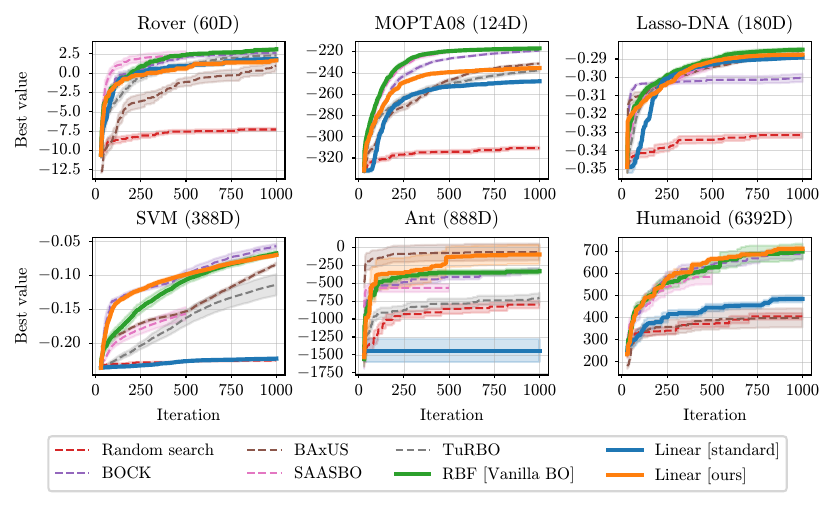}
    \end{minipage}
    \vrule
    \begin{minipage}[t]{0.24\textwidth}
        \centering
        \includegraphics[width=\textwidth]{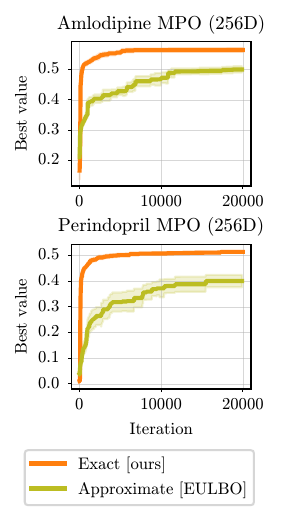}
    \end{minipage}
    \caption{{\bf Our linear kernel on spherically-mapped inputs matches state-of-the-art high-dimensional BO performance.} (Left): benchmarks with $N \approx D$ evaluation budgets. While standard linear kernels can fail to make any optimization progress, our modified kernel matches or exceeds competitive methods. (Right): benchmarks with $N \gg D$.
    The natural scalability of linear kernels, coupled with the improved optimization performance afforded by our spherical mapping, yields new state-of-the-art results on large-$N$ tasks.
    }
    \label{fig:main-plot}
\end{figure*}

\paragraph{Exact Thompson Sampling.}
Thompson sampling \citep{thompson1933likelihood}, a common acquisition function for high-dimensional problems \citep[e.g.][]{eriksson2019scalable} and parallel acquisitions \citep[e.g.][]{hernandez2017parallel}, requires sampling a realization of the GP posterior. While standard kernels necessitate approximations or discretizations of $\mathcal X$ for sampling, linear kernels yield exact posterior function samples. Functions under our linear model posterior can be represented as $f(\mathbf x) = \theta_0 + \boldsymbol \theta^\top P(\mathbf z)$, where $\theta_0$ and $\boldsymbol \theta$ are (finite-dimensional) Gaussian random variables. Hence, we can exactly sample functions by sampling $\theta_0$ and $\boldsymbol \theta$.

\section{\bfseries\small EXPERIMENTS} \label{sec:experiments}
To showcase the performance and scalability of our proposed linear kernel, we evaluate it on a range of optimization problems (both $N \approx D$ and $N \gg D$). Moreover, we ablate the design choices that differentiate our kernel from the standard linear kernel.

\paragraph{Benchmarks.} To test performance in the $N \approx D$ setting, we evaluate our method on a range of standard benchmarks from the HDBO literature ranging from $D = 60$ all the way to $D = 6392$: \texttt{Rover} \citep[60D;][]{wang2018batched}, \texttt{MOPTA08} \citep[124D;][]{eriksson2021high}, \texttt{Lasso-DNA} \citep[180D;][]{nardi2022lassobench}, \texttt{SVM} \citep[388D;][]{eriksson2021high}, \texttt{Ant} \citep[888D;][]{wang2020learning}, and \texttt{Humanoid} \citep[6392D;][]{wang2020learning}.
We evaluate these tasks under a budget of $N=1000$ observations. For the $N \gg D$ setting, we benchmark on multiple molecular tasks from the \texttt{GuacaMol} benchmark \citep{brown2019guacamol}, which typically require large optimization budgets for meaningful progress. We perform BO in the 256D latent space of a pre-trained variational auto-encoder \citep{maus2022local, maus2024approximation} with a $N = 20,\!000$ budget.

\paragraph{Baselines.} For the $N \approx D$ tasks, we compare against popular HDBO methods: \texttt{BOCK} \citep{oh2018bock}, \texttt{TuRBO} \citep{eriksson2019scalable}, \texttt{SAASBO} \citep{eriksson2021high}, \texttt{BAxUS} \citep{papenmeier2022increasing}, and \say{Vanilla BO} \citep{hvarfner2024vanilla}, as well as a standard linear kernel as in \Cref{eqn:std_linear_kernel}. The $N \gg D$ tasks necessitate the use of scalable GP approximations (except in the case of linear kernels), and therefore we compare against the \texttt{EULBO} method \citep{maus2024approximation}, which relies on variational approximations. For the \say{random search} baseline, we perform quasi-random Sobol sampling \citep{sobol1967distribution} of $\mathcal X$.

\paragraph{Experimental Set-Up.} Our set-up largely follows \citet{hvarfner2024vanilla}: we use the \texttt{LogEI} acquisition function \citep{ament2023unexpected}, initialize all methods with 30 quasi-random observations ($100$ points for $N \gg D$ tasks), and follow their other modeling protocols. We run each method for at least 10 different random seeds, and the plots in this paper display the mean and standard error of the mean (with respect to the seeds) as a solid/dotted line and shaded area, respectively. Details of our experimental set-up can be found in \Cref{app:implementation-details}.

\subsection{$N \approx D$ Optimization Problems} \label{sec:n=d}
\Cref{fig:main-plot} (left) demonstrates optimization performance of all methods on $N \approx D$ tasks. As expected, the standard linear kernel fails to optimize many benchmarks, though it obtains decent performance on the lower-dimensional tasks. In contrast, our modified linear model matches the current state-of-the-art on all $N \approx D$ benchmarks. In fact, the optimization trajectories from our linear kernel and Vanilla BO are statistically indistinguishable on the \texttt{Lasso-DNA} and \texttt{Humanoid} benchmarks. We note the stark performance difference between the modified and standard linear kernels on the \texttt{SVM} and \texttt{Ant} datasets. Moreover, in \Cref{app:guacamol}, we compare against Vanilla BO on a number of latent-space molecular tasks with a similar observation budget. Spherical linear kernels strongly outperform Vanilla BO on all these benchmarks, potentially suggesting our method is particularly advantageous for latent-space BO.

\subsection{$N \gg D$ Optimization Problems} \label{sec:n>>d}
For $N \gg D$ molecular tasks, we observe significant performance gains over the scalable \texttt{EULBO} method (\Cref{fig:main-plot}, right). We hypothesize that our linear kernel, unencumbered by the need for approximations, is better suited for large-$N$ tasks, despite its lack of representational capacity.

\begin{figure}[t!]
    \centering
    \includegraphics[width=\columnwidth]{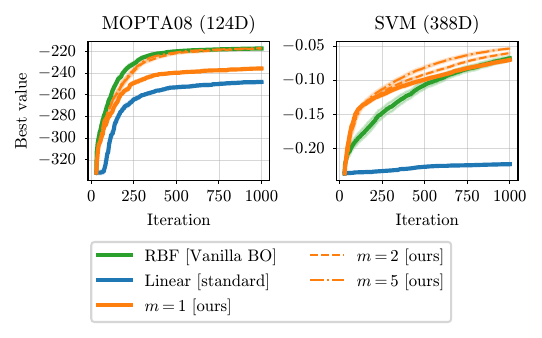}
    \caption{Extension of our spherical linear kernel to order-$m$ polynomial kernels.
    Higher-order polynomials do not improve much upon our linear model ($m=1$). See \Cref{fig:poly-full} for additional datasets.}
    \label{fig:poly-small}
\end{figure}

\subsection{Ablation Studies} \label{sec:ablations}
\paragraph{Higher-Order Polynomials.}
In \Cref{fig:poly-small} (and its extended version, \Cref{fig:poly-full}, in \Cref{app:poly}), we replicate the $N \approx D$ benchmark results for order-$m$ polynomial kernels of the form in \Cref{eq:poly-kernel}. (Recall that our proposed linear kernel is the special case of this form for $m=1$.) Surprisingly, the polynomial order has almost no effect on optimization performance, with our linear kernel matching the $m=5^\mathrm{th}$ order polynomial on nearly all benchmarks. From a practical perspective, higher order polynomials are net detrimental, as they lose scalability and exact Thompson sampling (the $O(D^m)$ feature representation of these kernels is too large to afford these benefits), without improving optimization performance.

\paragraph{Choice of Spherical Mapping.}
Our biggest modification to the standard linear kernel is the spherical mapping (see \Cref{sec:prop-lin-model}), i.e.\ the inverse stereographic projection defined in \Cref{eq:inv-stereo}. In \Cref{fig:spher-proj-small} (and its extended version, \Cref{fig:spher-proj-full}, in \Cref{app:spher-proj}), we compare a range of other common mappings from $\mathcal X$ to the hypersphere (see \Cref{app:spher-proj} for descriptions). We also compare against the case of no projection (None), but where we keep our other proposed modifications to the linear kernel (decoupled lengthscales and no implicit outputscale). While we can conclude that a spherical projection is necessary to improve performance, the choice of spherical projection itself can also have a drastic impact on performance.

We note that the inverse stereographic projection, which consistently outperforms all other mappings, is the only mapping that does not modify unit-norm inputs (i.e.\ $P(\mathbf z) = [\mathbf z, 0]$ for all $\Vert \mathbf z \Vert = 1$). We hypothesize that this property may be instrumental in its performance (see \Cref{sec:analysis}). Interestingly, while a spherical mapping is crucial for linear BO, it has almost no effect for non-parametric models. In \Cref{app:rbf-sphere}, we demonstrate that Vanilla BO performance is largely unchanged under the inverse stereographic projection.

\paragraph{Other Ablations.}
In \Cref{app:other-abl.}, we ablate over our other modifications (decoupled lengthscale, centering of the search space, implicit outputscale, etc.), showing that our decoupled parameterization makes our kernel less sensitive to any prior placed on the lengthscale vector $\boldsymbol \ell$. In contrast, priors on standard (coupled) lengthscales have a significant impact on BO performance \citep{hvarfner2024vanilla}.

\begin{figure}[t!]
    \centering
    \includegraphics[width=\columnwidth]{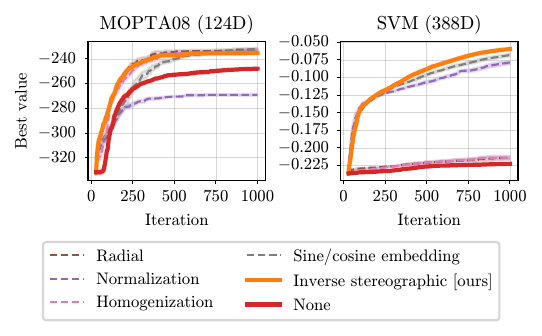}
    \caption{Ablation over the spherical mapping function used by our modified linear kernel. While most spherical mappings improve upon the unmodified inputs (red line), the inverse stereographic projection (orange line) outperforms all other mappings. See \Cref{fig:spher-proj-full} for additional datasets.}
    \label{fig:spher-proj-small}
\end{figure}

\begin{figure*}[t!]
    \centering
    \includegraphics[width=\textwidth]{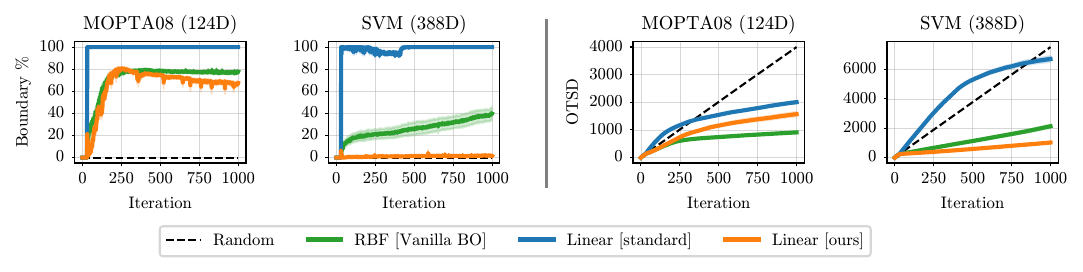}
    \caption{
    {\bf Differences in the acquired inputs of standard versus modified linear models.}
    (Left): the ``boundary \%'' depicts, for each acquisition $\mathbf x_t$, how many dimensions of the vector lie on the boundary of $\mathcal X$ (i.e. $\pm 1$).
    Standard linear models lead to acquisitions with points nearly in the corners of the hypercube (i.e.\ $100\%$ of dimensions on the boundary).
    Our linear model acquires non-corner points ($\approx 75\%$ boundary on \texttt{MOPTA08}) and interior points ($\approx 0\%$ boundary on \texttt{SVM}).
    See \Cref{fig:boundary-full}, \Cref{app:boundary-plot} for additional datasets.
    (Right): length of the shortest path connecting all acquired observations, or Observation Traveling Salesman Distance (OTSD).
    The corner-searching behavior of the standard linear model leads to acquisitions that are spread out over $\mathcal X$, whereas our modified linear model yields locality that is comparable to RBF-based models.
    See \Cref{fig:otsd-full}, \Cref{app:otsd} for additional datasets.
    }
    \label{fig:otsd-boundary}
\end{figure*}

\section{\bfseries\small ANALYSIS} \label{sec:analysis}
By conventional wisdom, linear models should not work as well as they do in Section~\ref{sec:experiments}. Their success challenges our understanding of what makes effective high-dimensional surrogate models. While our analysis cannot resolve this puzzle, it reveals a tension that may be central to explaining their efficacy. Reconciling the following contradictory observations highlights the non-intuitive nature of HDBO, suggesting that our understanding is fundamentally incomplete:
\begin{enumerate}
    \item after applying our spherical mapping, the linear kernel exhibits BO behavior that resembles more traditional non-parametric GPs; and yet
    \item the spherical mapping is geometrically minor for most points in high dimensional spaces, preserving local structure except at the boundary.
\end{enumerate}

\subsection{Spherical Linear Kernels Avoid the Pathologies of Standard Linear Kernels}
Standard linear kernels suffer from a pathological limitation in BO: as we prove in \Cref{sec:prop-lin-model}, they exclusively acquire points on the boundary of $\mathcal{X}$, implicitly discarding the interior of the search space. We empirically validate this in \Cref{fig:otsd-boundary} (left), which plots the percentage of entries of $\mathbf{x}_t$ equal to $\pm 1$ for each BO iteration $t$ under various kernels ($0\%$ indicates an interior point and $100\%$ indicates a corner). On the \texttt{MOPTA08} and \texttt{SVM} tasks, we observe that the standard linear kernel acquires points with $100\%$ of dimensions equal to $\pm 1$, i.e.\ corners of $\mathcal X$. This represents an even more extreme pathology than simple boundary-seeking: the model restricts itself to the $2^D$ corners, a measure-zero subset of both the boundary and the search space (i.e.\ a finite instead of infinite set of points).

In contrast, RBF kernels acquire points with $\leq 50\%$ of entries equal to $\pm 1$. While these points still lie on the boundary, they are less extreme and cover a broader region of $\mathcal{X}$. Our spherically-mapped linear kernel exhibits similar behavior to the RBF kernel, and even yields points entirely in the interior ($0\%$ of $\pm 1$ elements) on the \texttt{SVM} task. (See \Cref{app:boundary-plot} for plots on additional datasets.) Again, the spherical mapping prevents the monotonic growth in posterior statistics that drives standard linear models toward extremal points (see counterexample in \Cref{app:no-boundary}).

The Observation Traveling Salesman Distance \citep[OTSD;][]{papenmeier2025exploring} measures the minimum path length connecting all acquired points, where lower OTSD values suggest more local search. In \Cref{fig:otsd-boundary} (right), standard linear kernels yield large OTSD on \texttt{MOPTA08} and \texttt{SVM}—often exceeding random search—because corner points (i.e.\ 100\% of dimensions on the boundary) are maximally far apart. In contrast, both our spherically-projected linear kernel and RBF kernels yield similarly low OTSD values, suggesting more local search strategies. (See \Cref{app:otsd} for plots on additional datasets.) Collectively, these results provide preliminary evidence that spherically-mapped linear kernels exhibit exploration patterns more similar to RBF kernels than to their unmapped linear counterparts, though complete characterization remains future work.

\subsection{Spherical and Standard Linear Kernels are Asymptotically The Same}
While the spherical mapping $P: \mathbb{R}^D \rightarrow \mathbb{S}^D$ is crucial for avoiding boundary-seeking pathology, it is a relatively minor transformation for most points in $\mathcal{X}$. The inverse stereographic projection is an identity mapping for unit-norm inputs: $P(\mathbf{z}) = [\mathbf{z}, 0]$ for all $\|\mathbf{z}\| = 1$. When lengthscales are set to $\sqrt{3/D}$ as suggested by recent work \citep{hvarfner2024vanilla, xu2025standard}, any $\mathbf{x}$ drawn uniformly from $\mathcal{X}$ yields a scaled vector $\mathbf{z} = [x_1/\ell_1, \ldots, x_D/\ell_D]$ with $\|\mathbf{z}\| \to 1$ as $D \to \infty$—a direct consequence of the law of large numbers (see \Cref{app:gaussian-annulus}). Thus, for most points in $\mathcal{X}$, the spherically-mapped linear model behaves nearly identically to a standard linear model and has no additional representational capacity.

Moreover, we emphasize that the inverse stereographic projection transforms the $D$-dimensional function space into a $(D+1)$-dimensional function space, where this one additional dimension represents less than a 1\% increase in $D$ for most considered benchmarks. In contrast, polynomial kernels rely on a $\mathcal{O}(D^m)$-dimensional function space at a minimum, and RBF or Mat\'ern kernels rely on an $\infty$-dimensional function space. As a result, although our linear model is \textit{technically} not linear in $\mathbf{x}$ (though it is in $P(\mathbf{x})$), we argue it is \textit{approximately} linear in $\mathbf{x}$.

To validate that the spherical mapping does not meaningfully increase model expressiveness on a typical regression task, we fit both models to $400$ points generated by a Sobol sequence over $\mathcal{X}$ and evaluate predictions on $100$ held-out Sobol points (concentrated around unit norm after scaling). In \Cref{fig:boxplot} (left), both the spherically-mapped and standard linear models achieve nearly identical predictive RMSE on most benchmarks, while the RBF model generally achieves lower error due to its higher representational capacity (see \Cref{fig:boxplot-full}, \Cref{app:boxplot} for results on other benchmarks).

\subsection{Model Expressiveness Does Not Predict BO Performance}
This similarity on standard regression tasks makes the performance gap in BO (see \Cref{fig:main-plot}) all the more puzzling: if the spherical and standard linear models have similar predictive capabilities, why do spherical mappings improve optimization so dramatically? The answer is that model expressiveness and predictive accuracy on \emph{random test points} do not necessarily translate to optimization performance with \emph{adaptively chosen data}.

\begin{figure}[t!]
    \centering
    \includegraphics[width=\columnwidth]{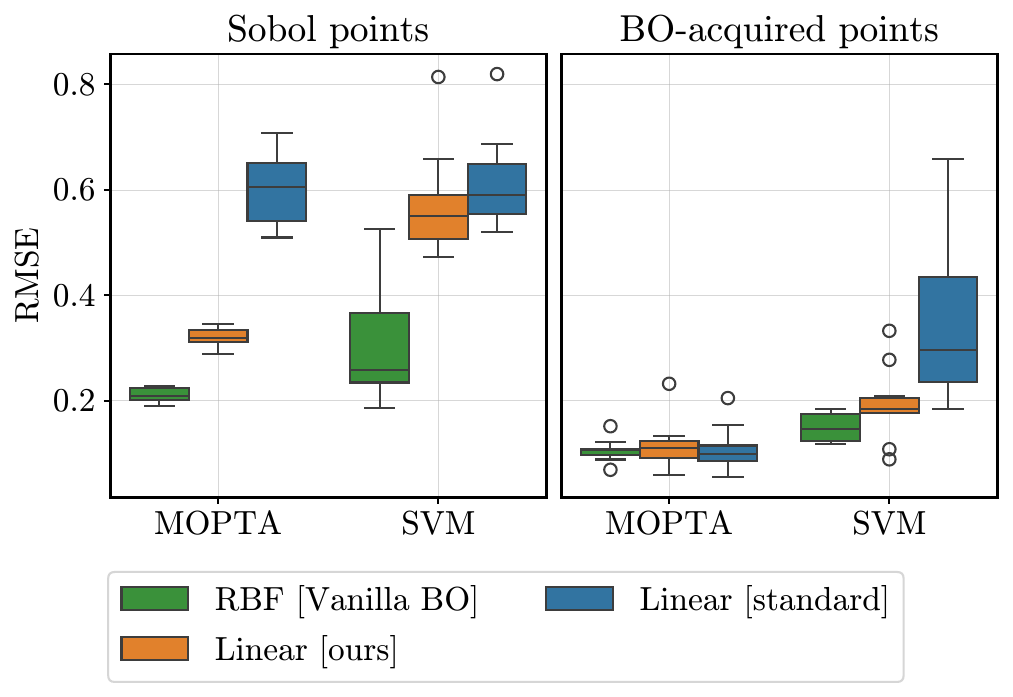}
    \caption{{\bf Spherical mappings affect BO performance, but not supervised regression performance.} (Left): predictive RMSE of GPs trained and tested on \emph{quasi-random data}. The standard and modified linear models make worse predictions than the RBF model.
    (Right): predictive RMSE on \emph{adaptively-chosen BO acquisitions}.
    Unlike with random data, linear models match the ability of RBF models for predictions at future BO acquisitions from prior ones. See \Cref{fig:boxplot-full}, \Cref{app:boxplot} for additional benchmarks.}
    \label{fig:boxplot}
\end{figure}

To further illustrate this point, \Cref{fig:boxplot} (right) repeats the above experiment on adaptively chosen BO acquisitions rather than Sobol points. We fit models to $400$ points from a BO trajectory (generated by our spherical linear model) and predict the objective values at the next $100$ acquired points. In \Cref{fig:boxplot} (right) and \Cref{fig:boxplot-full} (right), all three models—spherical linear, standard linear, and RBF—achieve similar predictive accuracy, despite their vastly different performance as BO surrogates (see \Cref{fig:boxplot-full}, right, in \Cref{app:boxplot} for results on other benchmarks). These results contrast the supervised learning setting where RBF substantially outperforms linear models on a number of problems.

Why are spherical linear models able to accurately predict future (adaptive) acquisitions but not random test data? One possible explanation is that acquisition functions like EI tend to select points locally \citep{garnett2023bayesian}, where objectives may be well-approximated by linear models (i.e.\ a first-order Taylor approximation) even when globally complex. However, fully characterizing which properties actually matter for HDBO success---beyond standard notions of expressiveness or generalization---remains an important open question.

\section{\bfseries\small CONCLUSION}
In this work, we show that linear models---arguably the simplest of all models---match state-of-the-art performance on high-dimensional BO tasks after mapping inputs to the unit sphere. From a {\bf practical perspective}, our modified linear kernel also provides state-of-the-art performance on large-$N$ tasks and promising performance on latent-space tasks. We thus recommend its adoption in these settings. From a {\bf theoretical perspective}, our findings call into question long-standing intuitions and practices. The fact that linear kernels, largely dismissed by the BO community, achieve such strong performance provides a compelling case against model complexity in high-dimensional settings. We hope that our initial analysis provides the seeds for radically new theories that unlock a new and more robust understanding of high-dimensional BO.

\subsubsection*{Acknowledgments}
Resources used in preparing this research were provided, in part, by the Province of Ontario, the Government of Canada through CIFAR, and companies sponsoring the Vector Institute. DF and GP acknowledge the support of the Natural Sciences and Engineering Research Council and the Social Sciences and Humanities Research Council of Canada (NSERC: RGPIN-2024-06405, NFRFE-2024-00830). NM was supported by NSF Graduate Research Fellowship (NSF GRFP). JRG was supported by NSF grants DBI-2400135 and IIS-2145644. HM was supported by Schmidt Sciences and Research England under the Expanding Excellence in England (E3) funding stream. GP is supported by the Canada CIFAR AI Chairs program.

\bibliography{references}

\begin{thebibliography}{}

\bibitem[Ament et~al., 2023]{ament2023unexpected}
Ament, S., Daulton, S., Eriksson, D., Balandat, M., and Bakshy, E. (2023).
\newblock Unexpected improvements to expected improvement for {B}ayesian
  optimization.
\newblock {\em Advances in Neural Information Processing Systems}.

\bibitem[Balandat et~al., 2020]{balandat2020botorch}
Balandat, M., Karrer, B., Jiang, D., Daulton, S., Letham, B., Wilson, A.~G.,
  and Bakshy, E. (2020).
\newblock {B}o{T}orch: A framework for efficient {M}onte-{C}arlo {B}ayesian
  optimization.
\newblock {\em Advances in Neural Information Processing Systems}.

\bibitem[Brown et~al., 2019]{brown2019guacamol}
Brown, N., Fiscato, M., Segler, M.~H., and Vaucher, A.~C. (2019).
\newblock {GuacaMol}: benchmarking models for de novo molecular design.
\newblock {\em Journal of Chemical Information and Modeling}.

\bibitem[Bull, 2011]{bull2011convergence}
Bull, A.~D. (2011).
\newblock Convergence rates of efficient global optimization algorithms.
\newblock {\em Journal of Machine Learning Research}.

\bibitem[Eriksson and Jankowiak, 2021]{eriksson2021high}
Eriksson, D. and Jankowiak, M. (2021).
\newblock High-dimensional {B}ayesian optimization with sparse axis-aligned
  subspaces.
\newblock In {\em Conference on Uncertainty in Uncertainty in Artificial
  Intelligence}.

\bibitem[Eriksson et~al., 2019]{eriksson2019scalable}
Eriksson, D., Pearce, M., Gardner, J., Turner, R.~D., and Poloczek, M. (2019).
\newblock Scalable global optimization via local {Bayesian} optimization.
\newblock {\em Advances in Neural Information Processing Systems}.

\bibitem[Fan et~al., 2024]{fan2024minimizing}
Fan, Z., Wang, W., Ng, S.~H., and Hu, Q. (2024).
\newblock Minimizing {UCB}: a better local search strategy in local {B}ayesian
  optimization.
\newblock {\em Advances in Neural Information Processing Systems}.

\bibitem[Frazier, 2018]{frazier2018bayesian}
Frazier, P.~I. (2018).
\newblock {B}ayesian optimization.
\newblock In {\em Recent Advances in Optimization and Modeling of Contemporary
  Problems}. Informs.

\bibitem[Gardner et~al., 2017]{gardner2017discovering}
Gardner, J., Guo, C., Weinberger, K., Garnett, R., and Grosse, R. (2017).
\newblock Discovering and exploiting additive structure for {B}ayesian
  optimization.
\newblock In {\em International Conference on Artificial Intelligence and
  Statistics}.

\bibitem[Gardner et~al., 2018]{gardner2018gpytorch}
Gardner, J., Pleiss, G., Weinberger, K.~Q., Bindel, D., and Wilson, A.~G.
  (2018).
\newblock {GPyTorch}: Blackbox matrix-matrix {G}aussian process inference with
  {GPU} acceleration.
\newblock {\em Advances in Neural Information Processing Systems}.

\bibitem[Garnett, 2023]{garnett2023bayesian}
Garnett, R. (2023).
\newblock {\em {B}ayesian optimization}.
\newblock Cambridge University Press.

\bibitem[Genton, 2001]{genton2001classes}
Genton, M.~G. (2001).
\newblock Classes of kernels for machine learning: a statistics perspective.
\newblock {\em Journal of Machine Learning Research}.

\bibitem[Han et~al., 2021]{han2021high}
Han, E., Arora, I., and Scarlett, J. (2021).
\newblock High-dimensional {B}ayesian optimization via tree-structured additive
  models.
\newblock In {\em Proceedings of the AAAI Conference on Artificial
  Intelligence}.

\bibitem[Hartman, 1973]{hartman1973some}
Hartman, J.~K. (1973).
\newblock Some experiments in global optimization.
\newblock {\em Naval Research Logistics Quarterly}.

\bibitem[Hellsten et~al., 2025]{hellsten2025leveraging}
Hellsten, E., Hvarfner, C., Papenmeier, L., and Nardi, L. (2025).
\newblock Leveraging axis-aligned subspaces for high-dimensional {B}ayesian
  optimization with group testing.
\newblock {\em arXiv preprint arXiv:2504.06111}.

\bibitem[Hensman et~al., 2013]{hensman2013gaussian}
Hensman, J., Fusi, N., and Lawrence, N.~D. (2013).
\newblock {G}aussian processes for big data.
\newblock In {\em Conference on Uncertainty in Artificial Intelligence}.

\bibitem[Hern{\'a}ndez-Lobato et~al., 2017]{hernandez2017parallel}
Hern{\'a}ndez-Lobato, J.~M., Requeima, J., Pyzer-Knapp, E.~O., and
  Aspuru-Guzik, A. (2017).
\newblock Parallel and distributed {T}hompson sampling for large-scale
  accelerated exploration of chemical space.
\newblock In {\em International Conference on Machine Learning}.

\bibitem[Hilbert and Cohn-Vossen, 2021]{hilbert2021geometry}
Hilbert, D. and Cohn-Vossen, S. (2021).
\newblock {\em Geometry and the Imagination}, volume~87.
\newblock American Mathematical Soc.

\bibitem[Hvarfner et~al., 2024]{hvarfner2024vanilla}
Hvarfner, C., Hellsten, E.~O., and Nardi, L. (2024).
\newblock Vanilla {B}ayesian optimization performs great in high dimensions.
\newblock In {\em International Conference on Machine Learning}.

\bibitem[Jones et~al., 1998]{jones1998efficient}
Jones, D.~R., Schonlau, M., and Welch, W.~J. (1998).
\newblock Efficient global optimization of expensive black-box functions.
\newblock {\em Journal of Global optimization}.

\bibitem[Kandasamy et~al., 2015]{kandasamy2015high}
Kandasamy, K., Schneider, J., and P{\'o}czos, B. (2015).
\newblock High dimensional {B}ayesian optimisation and bandits via additive
  models.
\newblock In {\em International Conference on Machine Learning}.

\bibitem[Kirschner et~al., 2019]{kirschner2019adaptive}
Kirschner, J., Mutny, M., Hiller, N., Ischebeck, R., and Krause, A. (2019).
\newblock Adaptive and safe {B}ayesian optimization in high dimensions via
  one-dimensional subspaces.
\newblock In {\em International Conference on Machine Learning}.

\bibitem[Letham et~al., 2020]{letham2020re}
Letham, B., Calandra, R., Rai, A., and Bakshy, E. (2020).
\newblock Re-examining linear embeddings for high-dimensional {B}ayesian
  optimization.
\newblock {\em Advances in Neural Information Processing Systems}.

\bibitem[Levy and Montalvo, 1985]{levy1985tunneling}
Levy, A.~V. and Montalvo, A. (1985).
\newblock The tunneling algorithm for the global minimization of functions.
\newblock {\em SIAM Journal on Scientific and Statistical Computing}.

\bibitem[Li et~al., 2017]{li2017high}
Li, C., Gupta, S., Rana, S., Nguyen, V., Venkatesh, S., and Shilton, A. (2017).
\newblock High dimensional {B}ayesian optimization using dropout.
\newblock In {\em International Joint Conference on Artificial Intelligence}.

\bibitem[Maus et~al., 2022]{maus2022local}
Maus, N., Jones, H., Moore, J., Kusner, M.~J., Bradshaw, J., and Gardner, J.
  (2022).
\newblock Local latent space {B}ayesian optimization over structured inputs.
\newblock {\em Advances in Neural Information Processing Systems}.

\bibitem[Maus et~al., 2024]{maus2024approximation}
Maus, N., Kim, K., Pleiss, G., Eriksson, D., Cunningham, J.~P., and Gardner,
  J.~R. (2024).
\newblock Approximation-aware {B}ayesian optimization.
\newblock {\em Advances in Neural Information Processing Systems}.

\bibitem[Mo{\v{c}}kus, 1974]{movckus1974bayesian}
Mo{\v{c}}kus, J. (1974).
\newblock On {B}ayesian methods for seeking the extremum.
\newblock In {\em IFIP Technical Conference on Optimization Techniques}.

\bibitem[Moriconi et~al., 2020]{moriconi2020high}
Moriconi, R., Deisenroth, M.~P., and Sesh~Kumar, K. (2020).
\newblock High-dimensional {B}ayesian optimization using low-dimensional
  feature spaces.
\newblock {\em Machine Learning}.

\bibitem[Moss et~al., 2023]{moss2023inducing}
Moss, H.~B., Ober, S.~W., and Picheny, V. (2023).
\newblock Inducing point allocation for sparse {G}aussian processes in
  high-throughput {B}ayesian optimisation.
\newblock In {\em International Conference on Artificial Intelligence and
  Statistics}.

\bibitem[M{\"u}ller et~al., 2021]{muller2021local}
M{\"u}ller, S., von Rohr, A., and Trimpe, S. (2021).
\newblock Local policy search with {B}ayesian optimization.
\newblock {\em Advances in Neural Information Processing Systems}.

\bibitem[Mutny and Krause, 2018]{mutny2018efficient}
Mutny, M. and Krause, A. (2018).
\newblock Efficient high dimensional {B}ayesian optimization with additivity
  and quadrature {F}ourier features.
\newblock {\em Advances in Neural Information Processing Systems}, 31.

\bibitem[Nardi et~al., 2022]{nardi2022lassobench}
Nardi, L., Gramfort, A., Salmon, J., and Sehic, K. (2022).
\newblock Lasso{B}ench: A high-dimensional hyperparameter optimization
  benchmark suite for lasso.
\newblock In {\em International Conference on Automated Machine Learning}.

\bibitem[Nayebi et~al., 2019]{nayebi2019framework}
Nayebi, A., Munteanu, A., and Poloczek, M. (2019).
\newblock A framework for {B}ayesian optimization in embedded subspaces.
\newblock In {\em International Conference on Machine Learning}.

\bibitem[Nguyen et~al., 2022]{nguyen2022local}
Nguyen, Q., Wu, K., Gardner, J., and Garnett, R. (2022).
\newblock Local {B}ayesian optimization via maximizing probability of descent.
\newblock {\em Advances in Neural Information Processing Systems}.

\bibitem[Oh et~al., 2018]{oh2018bock}
Oh, C., Gavves, E., and Welling, M. (2018).
\newblock {BOCK}: {B}ayesian optimization with cylindrical kernels.
\newblock In {\em International Conference on Machine Learning}.

\bibitem[Papenmeier et~al., 2025a]{papenmeier2025exploring}
Papenmeier, L., Cheng, N., Becker, S., and Nardi, L. (2025a).
\newblock Exploring exploration in {B}ayesian optimization.
\newblock In {\em Conference on Uncertainty in Artificial Intelligence}.

\bibitem[Papenmeier et~al., 2022]{papenmeier2022increasing}
Papenmeier, L., Nardi, L., and Poloczek, M. (2022).
\newblock Increasing the scope as you learn: Adaptive {B}ayesian optimization
  in nested subspaces.
\newblock {\em Advances in Neural Information Processing Systems}.

\bibitem[Papenmeier et~al., 2023]{papenmeier2023bounce}
Papenmeier, L., Nardi, L., and Poloczek, M. (2023).
\newblock Bounce: Reliable high-dimensional {B}ayesian optimization for
  combinatorial and mixed spaces.
\newblock {\em Advances in Neural Information Processing Systems}.

\bibitem[Papenmeier et~al., 2025b]{papenmeier2025understanding}
Papenmeier, L., Poloczek, M., and Nardi, L. (2025b).
\newblock Understanding high-dimensional {B}ayesian optimization.
\newblock In {\em International Conference on Machine Learning}.

\bibitem[Rashidi et~al., 2024]{rashidi2024cylindrical}
Rashidi, B., Johnstonbaugh, K., and Gao, C. (2024).
\newblock Cylindrical {T}hompson sampling for high-dimensional {B}ayesian
  optimization.
\newblock In {\em International Conference on Artificial Intelligence and
  Statistics}.

\bibitem[Rolland et~al., 2018]{rolland2018high}
Rolland, P., Scarlett, J., Bogunovic, I., and Cevher, V. (2018).
\newblock High-dimensional {B}ayesian optimization via additive models with
  overlapping groups.
\newblock In {\em International Conference on Artificial Intelligence and
  Statistics}.

\bibitem[Santoni et~al., 2024]{santoni2024comparison}
Santoni, M.~L., Raponi, E., Leone, R.~D., and Doerr, C. (2024).
\newblock Comparison of high-dimensional {B}ayesian optimization algorithms on
  {BBOB}.
\newblock {\em ACM Transactions on Evolutionary Learning}.

\bibitem[Schoenberg, 1942]{schoenberg1942positive}
Schoenberg, I. (1942).
\newblock Positive definite functions on spheres.
\newblock {\em Duke Mathematical Journal}.

\bibitem[Schoenberg, 1938]{schoenberg1938metric}
Schoenberg, I.~J. (1938).
\newblock Metric spaces and completely monotone functions.
\newblock {\em Annals of Mathematics}.

\bibitem[Shahriari et~al., 2015]{shahriari2015taking}
Shahriari, B., Swersky, K., Wang, Z., Adams, R.~P., and De~Freitas, N. (2015).
\newblock Taking the human out of the loop: A review of {B}ayesian
  optimization.
\newblock {\em Proceedings of the IEEE}.

\bibitem[Sobol, 1967]{sobol1967distribution}
Sobol, I.~M. (1967).
\newblock Distribution of points in a cube and approximate evaluation of
  integrals.
\newblock {\em USSR Computational mathematics and mathematical physics}.

\bibitem[Song et~al., 2022]{song2022monte}
Song, L., Xue, K., Huang, X., and Qian, C. (2022).
\newblock Monte {C}arlo tree search based variable selection for high
  dimensional {B}ayesian optimization.
\newblock {\em Advances in Neural Information Processing Systems}.

\bibitem[Srinivas et~al., 2010]{srinivas2010gaussian}
Srinivas, N., Krause, A., Kakade, S.~M., and Seeger, M. (2010).
\newblock {G}aussian process optimization in the bandit setting: No regret and
  experimental design.
\newblock In {\em International Conference on Machine Learning}.

\bibitem[Thompson, 1933]{thompson1933likelihood}
Thompson, W.~R. (1933).
\newblock On the likelihood that one unknown probability exceeds another in
  view of the evidence of two samples.
\newblock {\em Biometrika}.

\bibitem[Vakili et~al., 2021]{vakili2021scalable}
Vakili, S., Moss, H., Artemev, A., Dutordoir, V., and Picheny, V. (2021).
\newblock Scalable {T}hompson sampling using sparse {G}aussian process models.
\newblock {\em Advances in Neural Information Processing Systems}.

\bibitem[Vershynin, 2025]{vershynin2025high}
Vershynin, R. (2025).
\newblock {\em High-dimensional probability: An introduction with applications
  in data science}.
\newblock Cambridge University Press, second edition.

\bibitem[Wan et~al., 2021]{wan2021think}
Wan, X., Nguyen, V., Ha, H., Ru, B., Lu, C., and Osborne, M.~A. (2021).
\newblock Think global and act local: {B}ayesian optimisation over
  high-dimensional categorical and mixed search spaces.
\newblock In {\em International Conference on Machine Learning}.

\bibitem[Wang et~al., 2020]{wang2020learning}
Wang, L., Fonseca, R., and Tian, Y. (2020).
\newblock Learning search space partition for black-box optimization using
  {M}onte {C}arlo tree search.
\newblock {\em Advances in Neural Information Processing Systems}.

\bibitem[Wang et~al., 2018]{wang2018batched}
Wang, Z., Gehring, C., Kohli, P., and Jegelka, S. (2018).
\newblock Batched large-scale {B}ayesian optimization in high-dimensional
  spaces.
\newblock In {\em International Conference on Artificial Intelligence and
  Statistics}.

\bibitem[Wang et~al., 2016]{wang2016bayesian}
Wang, Z., Hutter, F., Zoghi, M., Matheson, D., and De~Feitas, N. (2016).
\newblock {B}ayesian optimization in a billion dimensions via random
  embeddings.
\newblock {\em Journal of Artificial Intelligence Research}.

\bibitem[Williams and Rasmussen, 2006]{williams2006gaussian}
Williams, C.~K. and Rasmussen, C.~E. (2006).
\newblock {\em {G}aussian processes for machine learning}.
\newblock MIT press.

\bibitem[Wilson et~al., 2020]{wilson2020efficiently}
Wilson, J., Borovitskiy, V., Terenin, A., Mostowsky, P., and Deisenroth, M.
  (2020).
\newblock Efficiently sampling functions from {G}aussian process posteriors.
\newblock In {\em International Conference on Machine Learning}.

\bibitem[Wu et~al., 2023]{wu2023behavior}
Wu, K., Kim, K., Garnett, R., and Gardner, J. (2023).
\newblock The behavior and convergence of local {B}ayesian optimization.
\newblock {\em Advances in Neural Information Processing Systems}.

\bibitem[Xu et~al., 2025]{xu2025standard}
Xu, Z., Wang, H., Phillips, J.~M., and Zhe, S. (2025).
\newblock Standard {G}aussian process is all you need for high-dimensional
  {B}ayesian optimization.
\newblock In {\em International Conference on Learning Representations}.

\bibitem[Ziomek and Ammar, 2023]{ziomek2023random}
Ziomek, J.~K. and Ammar, H.~B. (2023).
\newblock Are random decompositions all we need in high dimensional {B}ayesian
  optimisation?
\newblock In {\em International Conference on Machine Learning}.

\end{thebibliography}

\newpage

\appendix
\onecolumn
\aistatstitle{Appendix}

\section{Proofs, examples and derivations} \label{app:proofs}

\subsection{Proof of \Cref{thm:boundary}} \label{app:boundary}
\boundary*
\begin{proof}
    Most common acquisition functions, including Expected Improvement (\texttt{EI}) and Upper Confidence Bound (\texttt{UCB}), can be shown to be monotonically increasing in the posterior mean $\mu(\mathbf{x})$ and standard deviation $\sigma(\mathbf{x})$ of the underlying surrogate. Now, consider the following proof by contradiction: the acquisition optimum $\mathbf{x}$ does not lie on the boundary. If we parametrize $\mathbf{x}$ as a direction $\mathbf{z}$ and magnitude $c \geq 0$, we get
    \begin{equation*}
        \mathbf{x} = \underbrace{\lVert \mathbf{x} \rVert}_c \cdot \underbrace{\dfrac{\mathbf{x}}{\lVert \mathbf{x} \rVert}}_{\mathbf{z}}, \quad \text{with} \quad \mu(\mathbf{x}) = \mathbf{x}^\top \hat{\bm{\beta}} = c \cdot \mathbf{z}^\top \hat{\bm{\beta}} \quad \text{and} \quad \sigma(\mathbf{x}) = \sqrt{\sigma^2_\varepsilon + \mathbf{x}^\top \mathbf{S} \mathbf{x}} = \sqrt{\sigma^2_\varepsilon + c^2 \cdot \mathbf{z}^\top \mathbf{S} \mathbf{z}}.
    \end{equation*}
    As a result, for any $\mathbf{x}$, increasing $c$ monotonically increases $\mu(\mathbf{x})$ and $\sigma(\mathbf{x})$, and therefore also $\alpha(\mathbf{x})$. Thus, the interior point $\mathbf{x}$ cannot be the (constrained) optimum, since increasing $c$ will result in a higher acquisition value, which is a contradiction and concludes the proof. 
    
    Here, we have assumed $\mathbf{x}^\top \hat{\bm{\beta}} \geq 0$, since assuming otherwise would lead to a better optimum, namely $-\mathbf{x}$, which is also a contradiction. Moreover, unless all previous data points $\mathbf{x}_{1:t}$ obtain exactly the same objective value $y$ (which is highly unlikely), we can safely assume $\hat{\bm{\beta}} \neq (0, \ldots, 0)$, meaning that $\mu(\mathbf{x})$ will be strictly increasing in $c$. As a consequence, not only will there be an optimum on the boundary, but this optimum will also be unique, meaning there is no optimum in the interior of the hypercube.

    \paragraph{Adding an Intercept} To extend the proof to a Bayesian linear model with an intercept, we follow a similar approach as before, but need to pay closer attention to the parameterization of $\mathbf{x}$. Specifically, we write $\mu(\mathbf{x}) = \hat{\beta}_0 + \mathbf{x}^\top \hat{\bm{\beta}}$ using the augmented feature map $\varphi(\mathbf{x}) = [1;\mathbf{x}]$. The posterior variance can then be expressed as 
    \begin{equation*}
        \sigma^2(\mathbf{x}) = \sigma_\varepsilon^2 + \varphi(\mathbf{x})^\top \mathbf{S}\,\varphi(\mathbf{x}) = \sigma_\varepsilon^2 + S_{00} + 2\,\mathbf{x}^\top S_{0x} + \mathbf{x}^\top S_{xx}\mathbf{x}, \quad \text{with} \quad \mathbf{S}=\begin{bmatrix} S_{00} & S_{0x}^\top \\ S_{0x} & S_{xx} \end{bmatrix}.
    \end{equation*}
    Completing the square, we choose a shift point $\mathbf{x}_0$ that removes the linear term, i.e.\ any solution of $S_{xx}\mathbf{x}_0 = -S_{0x}$ (unique when $S_{xx}$ is invertible). Then
    \begin{equation*}
        \sigma^2(\mathbf{x}) \;=\; \underbrace{\sigma_\varepsilon^2 + S_{00} - \mathbf{x}_0^\top S_{xx}\mathbf{x}_0}_{\text{constant in }\mathbf{x}} \;+\; (\mathbf{x}-\mathbf{x}_0)^\top S_{xx}(\mathbf{x}-\mathbf{x}_0),
    \end{equation*}
    where $S_{xx}\succeq 0$ is the slope block of the posterior covariance.
    
    Now, we parametrize any interior point relative to $\mathbf{x}_0$ as $\mathbf{x}=\mathbf{x}_0+c\,\mathbf{z}$ with $c\ge 0$ and $\|\mathbf{z}\|=1$. If necessary, we flip $\mathbf{z}$ so that $\mathbf{z}^\top \hat{\bm{\beta}}\ge 0$. Along this ray, we obtain
    \begin{equation*}
        \mu(\mathbf{x}) = \hat{\beta}_0 + \mathbf{x}_0^\top\hat{\bm{\beta}} + c\,\mathbf{z}^\top\hat{\bm{\beta}},
        \quad \text{and} \quad
        \sigma(\mathbf{x}) = \sqrt{A + c^2\,\mathbf{z}^\top S_{xx}\mathbf{z}},
    \end{equation*}
    with $A=\sigma_\varepsilon^2 + S_{00} - \mathbf{x}_0^\top S_{xx}\mathbf{x}_0$ independent of $c$. Because $\mathbf{z}^\top\hat{\bm{\beta}}\ge 0$ and $S_{xx}\succeq 0$, both $\mu(\mathbf{x})$ and $\sigma(\mathbf{x})$ are nondecreasing in $c$, and are strictly increasing unless simultaneously $\mathbf{z}^\top\bm{\beta}=0$ and $\mathbf{z}^\top S_{xx}\mathbf{z}=0$ (a nongeneric degeneracy).
\end{proof}

\subsection{Counterexample to \Cref{thm:boundary} (For Spherical Projections)} \label{app:no-boundary} 
We provide a specific example to support the following statement from \Cref{sec:prop-lin-model}: \say{when points are projected to the hypersphere, \Cref{thm:boundary} does not apply anymore.} That is, for
\begin{equation*}
    \mathbf{x}_{t+1} = \arg \max_{\mathbf{x} \in [-1, 1]^D} \alpha_t\left(P(\mathbf{x})\right),
\end{equation*}
where $P: \mathbb{R}^D \rightarrow \mathbb{S}^D$, the optimum $\mathbf{x}_{t+1}$ can be any point in the hypercube (i.e.\ $\|\mathbf{x}_{t+1}\|_\infty \in [0,1]$), instead of only boundary points (i.e.\ $\|\mathbf{x}_{t+1}\|_\infty = 1$) as stated in \Cref{thm:boundary}.

For simplicity, consider $D=1$ and \texttt{UCB} with an exploration factor of $\lambda = 0$, giving us
\begin{equation*}
    \alpha_t(P(x)) = P(x)^\top \hat{\bm{\beta}}, \quad \text{where we assume} \quad \hat{\bm{\beta}} =  \begin{pmatrix} \sfrac{1}{2} \\
        -1
    \end{pmatrix}.
\end{equation*}
As spherical mapping $P: \mathbb{R}^D \rightarrow \mathbb{S}^D$, we consider the inverse stereographic projection from \Cref{eq:inv-stereo}, giving us
\begin{equation*}
    \alpha(x) = \dfrac{1}{2} \cdot \dfrac{2x}{x^2+1} - \dfrac{x^2-1}{x^2+1}, \quad \text{with} \quad \alpha(-1) = -\tfrac{1}{2}, \quad \alpha\left(\tfrac{1}{2}\right) =1, \quad \alpha(1) =\tfrac{1}{2}.
\end{equation*}
Here, the interior point $x=\sfrac{1}{2}$ leads to a higher acquisition value than the two boundary points $x = -1$ and $x=1$, thus representing a clear example for which \Cref{thm:boundary} does not hold (because of the non-linear projection $P(x)$ from the real line $\mathbb{R}$ to the circle $\mathbb{S}$).

\subsection{Isotropic Kernels Become Dot-Product Kernels (After Spherical Projection)} \label{app:isotropic-dot-product}
Any isotropic kernel becomes a dot product kernel when applied to unit norm inputs.

Let $k: \mathbb R^\infty \times \mathbb R^\infty \to \mathbb R$ be an isotropic kernel, where $\mathbb R^\infty \times \mathbb R^\infty$ denotes that the kernel is valid on any possible input dimensionality. By the result of \citet[][Theorem~2]{schoenberg1938metric}, this kernel can be expressed as:
\begin{equation}
    k(\mathbf x, \mathbf x') = \int \exp\left(- \frac{\ell}{2} \left\Vert \mathbf x - \mathbf x' \right\Vert_2^2\right) d \mu(\ell)
    = \int \exp\left(- \frac{\ell \left( \Vert \mathbf x \Vert^2 + \Vert \mathbf x' \Vert^2 \right)}{2} \right) \exp\left( \ell \mathbf x^\top \mathbf x' \right) d \mu(\ell)
    \label{eqn:shoenberg_isometric}
\end{equation}
where $\mu(\ell)$ is some positive finite measure over $[0, \infty)$. Without loss of generality, we assume that $\mu$ is a probability measure---alternatively, we can normalize the measure by $1/\mu([0, \infty))$---as the value of $\mu([0, \infty))$ can be viewed as an implicit outputscale on the kernel.

Applying this kernel to sphere-projected inputs $P(\mathbf z), P(\mathbf z')$ (which are both unit norm), we have:
\begin{equation}
    k(P(\mathbf z), P(\mathbf z')) = \int \exp\left(- \ell \right) \exp\left( \ell P(\mathbf z)^\top P(\mathbf z') \right) d \mu(\ell)
    \nonumber
\end{equation}
Taking a Taylor expansion of the right exponential and applying Fubini-Tonelli, we have:
\begin{equation}
    k(P(\mathbf z), P(\mathbf z')) = \sum_{i=0}^\infty \frac{\int \ell^i \exp\left(- \ell \right) d \mu(\ell)}{i!} \left[ P(\mathbf z)^\top P(\mathbf z') \right]^i.
    \label{eqn:shoenberg_taylor}
\end{equation}
As $\ell^i \exp\left(- \ell \right)$ is trivially bounded for any $i \in \mathbb N$ and $\ell \geq 0$, the integrals in \Cref{eqn:shoenberg_taylor} are all finite and well defined. Furthermore, by \Cref{eqn:shoenberg_isometric} we have that $k(P(\mathbf z), P(\mathbf z)) = 1$ for all $\mathbf z$; combining this fact with \Cref{eqn:shoenberg_taylor} we have that
\begin{equation}
k(P(\mathbf z), P(\mathbf z)) = \sum_{i=0}^\infty \frac{\int \ell^i \exp\left(- \ell \right) d \mu(\ell)}{i!} = 1.
\end{equation}
In other words, the polynomial coefficients sum to 1, and thus $k$ restricted to the unit sphere can be arbitrarily approximated by the higher order polynomial kernel in \Cref{eq:poly-kernel}.

\subsection{Derivation of \Cref{eq:taylor-exp}} \label{app:taylor-exp}
The polynomial expression of the sphere-based RBF kernel can be derived from \Cref{eqn:shoenberg_isometric,eqn:shoenberg_taylor}, under the atomic measure $d\mu(\ell) = \delta(\ell - 1)$, where $\delta$ is the Dirac delta function.

\subsection{Thin Shell Phenomenon} \label{app:gaussian-annulus}
Assume $X = (X_1, X_2, \ldots, X_D)$ is generated uniformly over the $D$-dimensional hypercube $[-1,1]^D$. Moreover, define the scaled variables $Z$ as $Z = \begin{bmatrix} X_1/\ell_1 & \cdots & X_D/\ell_D \end{bmatrix} = X \sqrt{3/D}$, where the lengthscales have been scaled according to $O(\sqrt{D})$ as suggested in recent works \citep{hvarfner2024vanilla, xu2025standard, papenmeier2025understanding}. Then, we have
\begin{equation*}
    \mathbb{E}\left(X_i^2\right) = \frac{1}{2}\int_{-1}^1 x_i^2 \, dx = \frac{1}{3}, \quad \text{and} \quad
    \textrm{var}\left(X_i^2\right) = \frac{1}{2}\int_{-1}^1 x_i^4 \, dx - \mathbb{E}\left(X_i^2\right)^2 = \frac{4}{45},
\end{equation*}
which gives us
\begin{equation*}
    \mathbb{E}\left(Z_i^2\right) = \frac{1}{D}, \quad \text{and} \quad
    \textrm{var}\left(Z_i^2\right) = \frac{4}{5D^2}.
\end{equation*}
Using the above, we have
\begin{equation*}
    \mathbb{E}\left(\|Z\|^2\right) = \sum_{i=1}^{D} \mathbb{E}\left(Z_i^2\right) = 1, \quad \text{and} \quad
    \mathrm{var}\left(\|Z\|^2\right) = \sum_{i=1}^{D} \textrm{var}\left(Z_i^2\right) = \frac{4}{5D}.
\end{equation*}
Therefore, as $D \rightarrow \infty$, it follows that
\begin{equation*}
    \mathbb{E}\left(\|Z\|^2\right) = 1, \quad \text{and} \quad
    \mathrm{var}\left(\|Z\|^2\right) = 0,
\end{equation*}
meaning $\|Z\|$ converges to 1 (in probability, and almost surely). This phenomenon is also known more broadly as the \say{thin shell phenomenon}, and we refer the reader to \cite{vershynin2025high} for a more complete treatment of the topic.

\section{Extended Related Work} \label{app:hdbo}
\paragraph{Low-Dimensional Subspaces}
Instead of searching in a high-dimensional space, a major stream of HDBO methods instead map $\mathcal{X}$ to a low-dimensional subspace, on which the search is more feasible. This mapping can be: \textit{(i) linear}, with methods such as \texttt{REMBO} \citep{wang2016bayesian}, \texttt{HeSBO} \citep{nayebi2019framework}, \texttt{ALEBO} \citep{letham2020re}, \texttt{BAxUS} \citep{papenmeier2022increasing}, and \texttt{Bounce}  \citep{papenmeier2023bounce}, or \textit{(ii) non-linear}, with methods such as \texttt{DMGPC-BO} \citep{moriconi2020high} and \texttt{LOL-BO} \citep{maus2022local}. Additionally, the linear embeddings can be focused on \textit{variable selection} only, as done in \texttt{Dropout-Mix} \citep{li2017high}, \texttt{SAASBO} \citep{eriksson2021high}, \texttt{MCTS-VS} \citep{song2022monte}, and \texttt{GTBO} \citep{hellsten2025leveraging}.

\paragraph{Local Search} Alternatively, instead of performing global search in a low-dimensional subspace, another stream of methods instead perform local search in the full-dimensional space. This can be done using: \textit{(i) trust regions}, such as \texttt{TuRBO} \citep{eriksson2019scalable}, \texttt{CASMOPOLITAN} \citep{wan2021think}, and \texttt{CTS-TuRBO} \citep{rashidi2024cylindrical}, or \textit{(ii) approximate gradients}, such as \texttt{GIBO} \citep{muller2021local, wu2023behavior}, \texttt{MPD} \citep{nguyen2022local}, and \texttt{LA-MinUCB} \citep{fan2024minimizing}.

\paragraph{Additive Decompositions} Rather than change the search space, additive-decomposition methods impose simplifying assumptions on the structure of the objective function. Examples include \texttt{Add-GP-UCB} \citep{kandasamy2015high}, \texttt{G-Add-GP-UCB} \citep{rolland2018high}, \texttt{Tree-GP-UCB} \citep{han2021high}, and \texttt{RDUCB} \citep{ziomek2023random}.

\section{Implementation Details} \label{app:implementation-details}
We run \texttt{EULBO} \citep{maus2024approximation} and \texttt{BAxUS} \citep{papenmeier2022increasing} using the authors' original implementations. For \texttt{TuRBO} \citep{eriksson2019scalable}, \texttt{SAASBO} \citep{eriksson2021high}, and Vanilla BO \citep{hvarfner2024vanilla}, we use implementations present in \texttt{BoTorch} \citep{balandat2020botorch}. For \texttt{BOCK} \citep{oh2018bock}, we replicate the authors' original implementation as closely as possible using the provided kernel in \texttt{GPyTorch} \citep{gardner2018gpytorch}, allowing for an easy integration with \texttt{BoTorch}.

Due to the high computational costs associated with \texttt{SAASBO}, we only run the model for 500 iterations, whereas all other methods are run for 1000 iterations. We run each method for at least 10 different random seeds, and the plots in this paper display the mean and standard error of the mean (with respect to the seeds) as a solid line and shaded area, respectively. We run our linear kernel as well other baselines on a cluster containing both CPUs and GPUs, where the latter were used to speed up computation for slower methods or datasets. Vanilla BO and our proposed linear kernel delivered similar run times ($\pm10\%$): given an evaluation budget of $N = 1000$, the pipeline takes between 1h and 12h on 1 NVIDIA GPU (e.g.\ L40S) depending on the benchmark. For the molecular experiments with $N>~$20,000 (on a similar GPU), the pipeline takes between four and six days.

\section{Additional Ablation Studies}

\subsection{Effect of Hyperprior} \label{app:hyperpriors}
As mentioned in \Cref{sec:prop-lin-model}, performance of our linear model is largely unaffected by the hyperprior placed on the lengthscales. In \Cref{fig:hyperpriors}, we evaluate three different hyperpriors:
\begin{enumerate}
\itemsep0em
    \item \texttt{Gamma}: $\ell_i \sim \Gamma(3,6)$, which used to be the default in \texttt{BoTorch} \citep{hvarfner2024vanilla},
    \item \texttt{LogNormal}: $\ell_i \sim \mathcal{LN}\left(\sqrt{2} + \frac{\log(D)}{2}, \sqrt{3}\right)$, which is used in Vanilla BO (often referred to as \texttt{DSP}), and
    \item \texttt{LogNormal}*: $\ell_i \sim \mathcal{LN}\left(\sqrt{2} + \frac{\log(1)}{2}, \sqrt{3}\right)$, which is equivalent to \texttt{DSP} without the dimensionality $D$, since scaling by $D$ is not needed when combined with our global lengthscale $a$ (see \Cref{sec:prop-lin-model}).
\end{enumerate}
As can be seen, using \texttt{Gamma} instead of \texttt{LogNormal}* leads to identical performance for our linear model (orange and red lines), whereas it results in large differences when applied to standard RBF kernels (green and purple lines). In the following section, we discuss how our proposed spherical projection and global lengthscale can be combined with RBF kernels (brown and pink lines).

\begin{figure*}[h!]
    \centering
    \includegraphics[width=\textwidth]{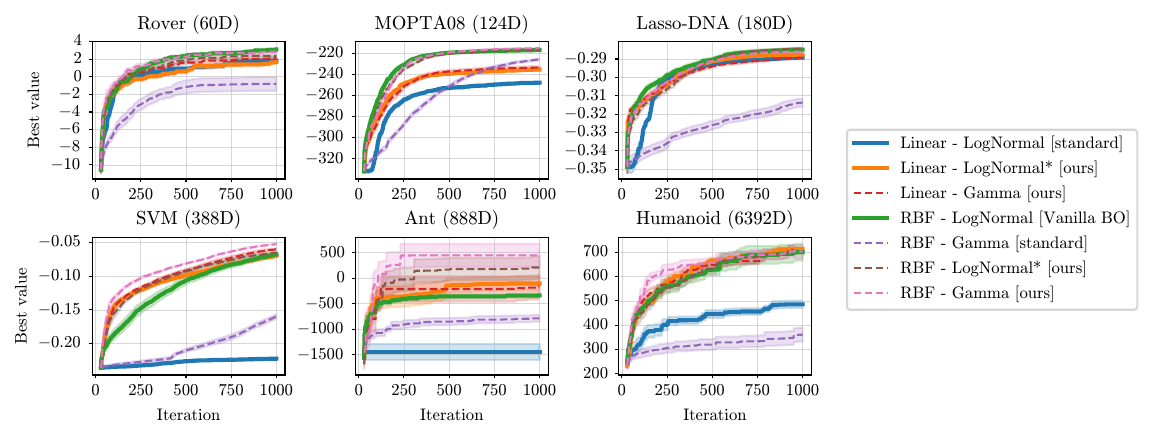}
    \caption{When applied to our spherical mapping with global lengthscale, both linear and RBF kernels are robust to the choice of (lengthscale) hyperprior, and match the performance of Vanilla BO.}
    \label{fig:hyperpriors}
\end{figure*}

\subsection{Effect of Our Spherical Mapping on an RBF Kernel} \label{app:rbf-sphere}
In \Cref{sec:ablations}, we claim that the inverse stereographic projection has almost no effect for non-parametric models. To support this claim, we evaluate the RBF kernel on points projected to the unit hypersphere; see the left-hand side of \Cref{eq:taylor-exp} for a more explicit formulation. In \Cref{fig:hyperpriors}, this RBF kernel applied to the hypersphere (brown line) leads to similar performance as the RBF kernel from Vanilla BO (green line), supporting our claim above. Moreover, replacing the \texttt{LogNormal} prior from Vanilla BO by the \texttt{Gamma} prior (pink line) does not degrade performance for our model (thanks to our projection and global lengthscale $a$), whereas it does lead to significantly worse results for the standard RBF kernel \citep[as presented in][]{hvarfner2024vanilla}.

\subsection{Effects of Projection, Centering, ARD, and Global Lengthscale} \label{app:other-abl.}
In \Cref{fig:ablations}, we modify certain components of our model proposed in \Cref{sec:prop-lin-model} and assess their impact on performance. We consider four modifications: \emph{(i)} no spherical projection $P$ (i.e.\ $P$ is an identity map), \emph{(ii)} a non-centered search space $[0,1]^D$ (instead of $[-1, 1]^D$), \emph{(iii)} no Automatic Relevance Determination (ARD) for the lengthscales (i.e.\ replacing $\ell_i$ by $\ell$), and \emph{(iv)} removing the global lengthscale $a$ (i.e.\ $a=1$).

\begin{figure*}[h!]
    \centering
    \includegraphics[width=\textwidth]{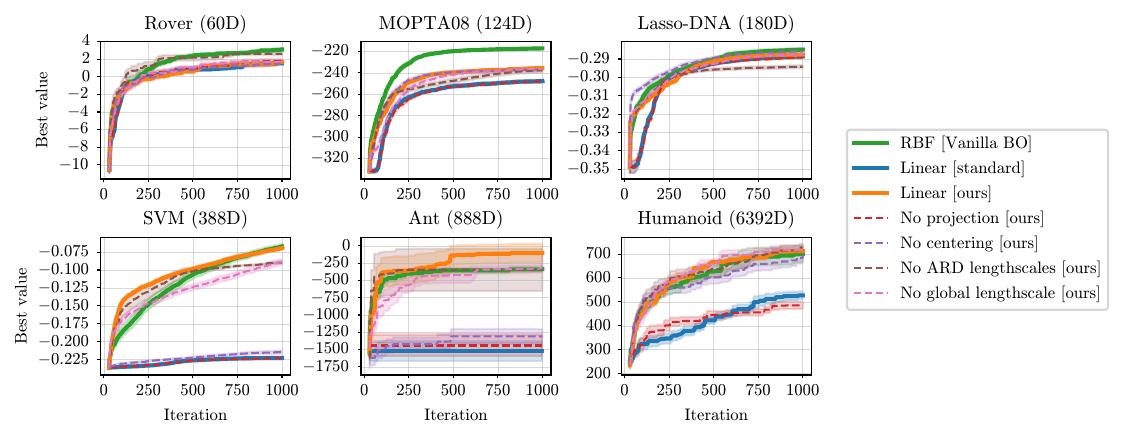}
    \caption{The spherical projection and centering of the hypercube search-space play a key role in our proposed linear model, whereas ARD and our global lengthscale only provide modest improvements.}
    \label{fig:ablations}
\end{figure*}

As confirmed by \Cref{fig:spher-proj-small,fig:spher-proj-full}, the spherical projection is the most important element among our different modifications, and removing it significantly affects performance. Similarly, going from a centered hypercube $[-1,1]^D$ to a non-centered hypercube $[0,1]^D$ leads to a noticeable drop in performance. We hypothesize that this is related to the inverse stereographic projection, where the \say{0-boundary} and the \say{1-boundary} are treated differently if we don't center, but exactly the same if we do. Specifically, \Cref{eq:inv-stereo} heavily relies on the norm of $\mathbf{x}$, and both boundaries will have different norms if no centering is applied. Interestingly, though both our global lengthscale and ARD bring performance improvements, these improvements are rather minor in contrast to centering and spherical projection of the search space.

\subsection{Effect of Acquisition Function}
Our results do not change significantly if we replace EI with UCB (\Cref{fig:ucb}): Vanilla BO and our proposed linear model obtain performance similar to that in \Cref{fig:main-plot}. If we replace EI with TS (\Cref{fig:ts_pathwise}), both methods incur a noticeable performance drop. However, our spherical linear model still matches, and sometimes exceeds, the performance of an RBF GP with TS (i.e.\ Vanilla BO). For the generation of the Thompson samples, we rely on \citet{wilson2020efficiently}, which means we can optimize the acquisition function using gradient-based methods.

\begin{figure*}[h!]
    \centering
    \includegraphics[width=\textwidth]{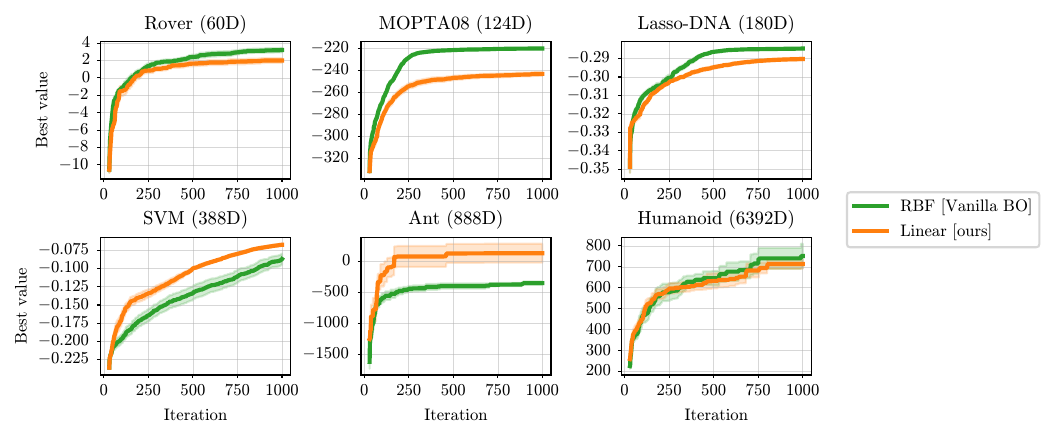}
    \caption{Replacing EI with UCB does not significantly affect performance of Vanilla BO and our proposed linear model.}
    \label{fig:ucb}
\end{figure*}

\begin{figure*}[h!]
    \centering
    \includegraphics[width=\textwidth]{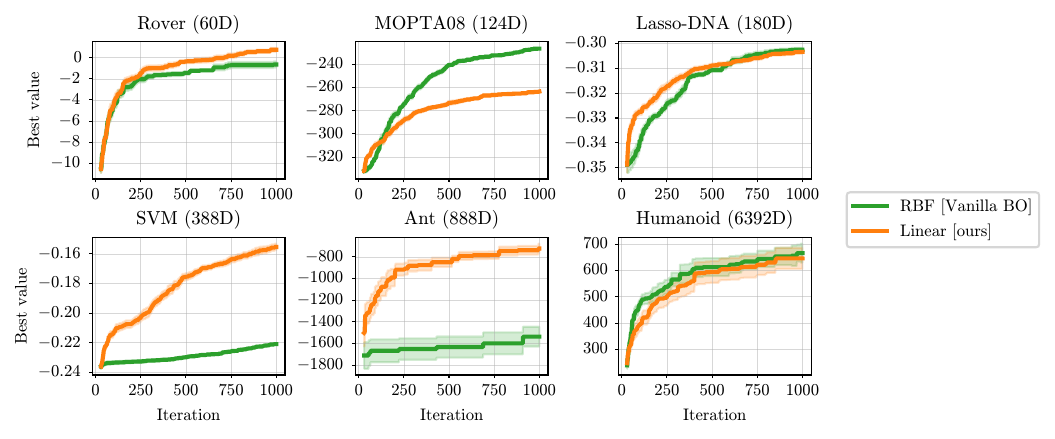}
    \caption{Replacing EI with TS leads to a noticeable performance drop for both methods, but overall our proposed linear model is still on par with (or maybe slightly better than) Vanilla BO.}
    \label{fig:ts_pathwise}
\end{figure*}

\subsection{Effect of Problem Structure}
Here we test our linear models versus Vanilla BO models on synthetic problems with known objective functions. In particular, we aim to assess the relative efficacy of our method as we increase the degree of sparsity and nonlinearity in the objective.

\paragraph{Effect of Sparsity.}
Following the setup in \citet[][Figure~5]{hvarfner2024vanilla}, we embed the \texttt{Levy} ($D=4$) and \texttt{Hartmann} ($D=6$) functions into a range of higher dimensionalities ($D=25$, $D=100$, $D=300$, and $D=1000$), thus leading to objective functions with increasing amounts of sparsity \citep{levy1985tunneling,hartman1973some}. For both test functions, we observe a clear trend across dimensionality (\Cref{fig:synthetic}): as the embedding dimension grows (i.e.\ sparsity increases), the gap between Vanilla BO and our proposed linear model vanishes. From these results, we conclude that our proposed linear model might be more suited to functions with a higher dimensionality or a higher degree of sparsity.

\begin{figure*}[h!]
    \centering
    \includegraphics[width=\textwidth]{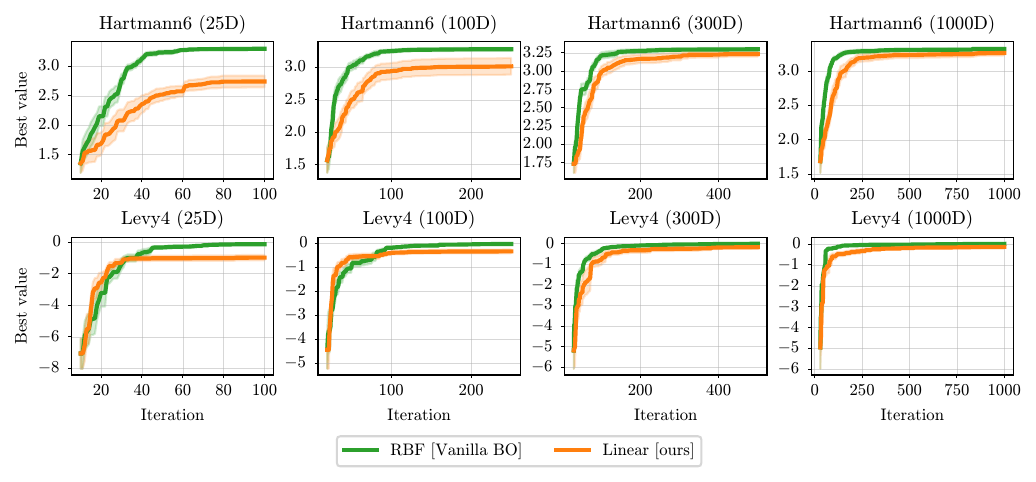}
    \caption{{\bf Impact of objective function sparsity on optimization results.}
    The Hartmann6 (top) and Levy4 (bottom) test functions are embedded in high-dimensional spaces, following the protocol of \citet{hvarfner2024vanilla}.
    Increasing the ambient dimensionality reduces the performance gap between our proposed linear model and Vanilla BO;
    however both methods are comparable in most settings.}
    \label{fig:synthetic}
\end{figure*}

\paragraph{Effect of Non-Linearity in the Objective Function.}
We consider objective functions that are varying convex combinations of linear and non-linear objective functions. Specifically, we draw the objective function from a Gaussian process with the following kernel:
\begin{equation} \label{eq:linear_rbf_mix}
    k_\textrm{test} = \alpha \, k_\textrm{linear} + (1-\alpha) \, k_\textrm{RBF}, \quad \text{with} \quad \alpha \in [0,1].
\end{equation}
Small values of $\alpha$ correspond to highly nonlinear functions; large values correspond to nearly linear functions. In \Cref{fig:linear_rbf_mix}, as $\alpha$ goes from zero to one, the difference between our proposed linear model and Vanilla BO vanishes (30\% gap to 5\% gap). Nevertheless, we observe that our spherical linear model still yields substantial optimization progress even when the objective function is entirely nonlinear ($\alpha=0$).

\begin{figure*}[h!]
    \centering
    \includegraphics[width=\textwidth]{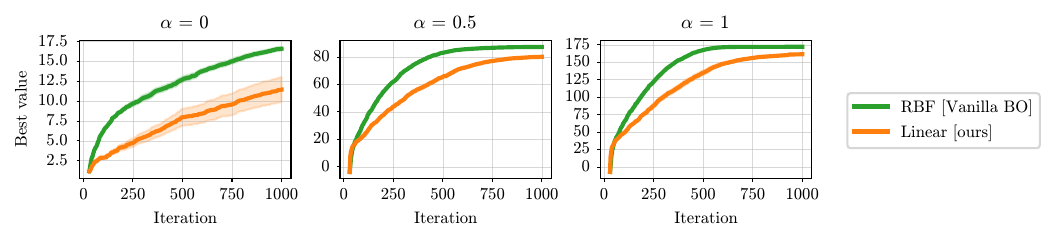}
    \caption{{\bf Impact of objective function nonlinearity on optimization results.}
    The objective function is drawn from a convex combination of linear and non-linear Gaussian processes,
    where $\alpha$ is the combination parameter.
    $\alpha=0$ (left) is entirely non-linear; $\alpha=1$ (right) is entirely linear.
    The difference between our proposed linear model and Vanilla BO decreases as $\alpha \to 1$, though performance is arguably comparable at all combination levels.
    }
    \label{fig:linear_rbf_mix}
\end{figure*}

\section{Complete Experimental Results}

\subsection{\texttt{GuacaMol} Datasets for $N=1,\!000$} \label{app:guacamol}
We compare our linear kernel from \Cref{eq:lin-kernel} against Vanilla BO on nine different tasks from the \texttt{GuacaMol} benchmark \citep{brown2019guacamol}. Specifically, we run both methods in the 256D latent space of the \texttt{SELFIES-VAE} \citep{maus2022local}, and plot the results in \Cref{fig:guacamol-1k}. On all nine datasets, our linear model significantly outperforms Vanilla BO, potentially indicating (a part of) its design is better suited for (molecular) latent spaces. We leave more thorough analysis of this strong performance to further work, but believe it to be a fruitful direction.

\begin{figure*}[h!]
    \centering
    \includegraphics[width=\textwidth]{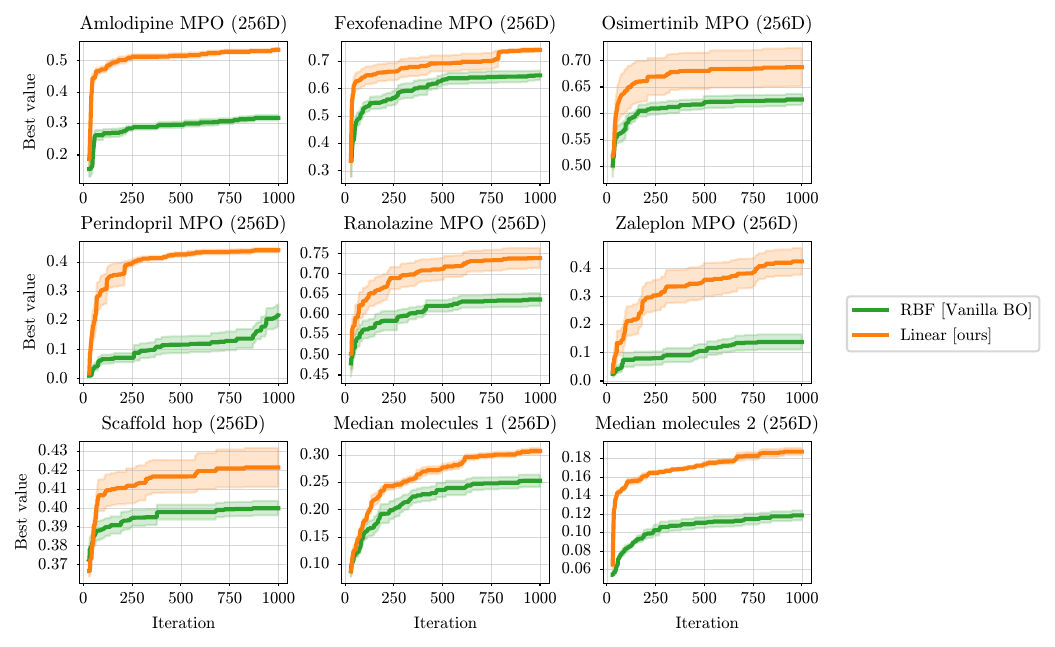}
    \caption{Our proposed linear kernel from \Cref{eq:lin-kernel} significantly outperforms Vanilla BO on molecular tasks from the \texttt{GuacaMol} benchmark \citep{brown2019guacamol}, where BO takes place in the 256D latent space of the \texttt{SELFIES-VAE} \citep{maus2022local}.}
    \label{fig:guacamol-1k}
\end{figure*}

\subsection{Higher-Order Polynomials} \label{app:poly}
We replicate the results of \Cref{fig:poly-small} in \Cref{fig:poly-full}, where the latter now contains six datasets instead of two. The conclusions for this extended figure are the same as in \Cref{sec:ablations}: higher-order polynomials ($m>1$) from \Cref{eq:poly-kernel} all obtain near-identical performance compared to the linear kernel ($m=1$). The only (minor) exception to this rule is the \texttt{MOPTA08} dataset, where higher-order polynomials slightly outperform the linear kernel.

\begin{figure*}[h!]
    \centering
    \includegraphics[width=\textwidth]{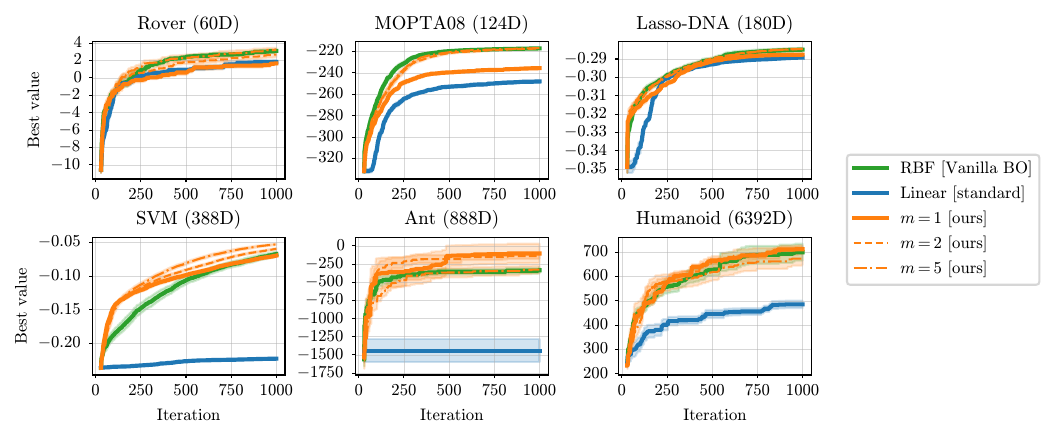}
    \caption{Extension of \Cref{fig:poly-small} (high-order polynomials do not improve upon our linear model).}
    \label{fig:poly-full}
\end{figure*}

\subsection{Choice of Spherical Mapping} \label{app:spher-proj}
We extend the findings from \Cref{fig:spher-proj-small} in \Cref{fig:spher-proj-full}, which expands the analysis to encompass six datasets rather than two. In both figures, we rely on our linear kernel from \Cref{eq:lin-kernel}, and only modify the spherical projection $P$. Here, \texttt{None} means no projection is used (i.e.\ $P$ is an identity map), the \texttt{Inverse}~\texttt{stereographic} projection is defined in \Cref{eq:inv-stereo}, and the remaining projections can be written as:
\begin{align*}
    P_\textrm{radial}(\mathbf{z}) &= \tfrac{1}{\gamma} \begin{bmatrix} z_1 & \cdots  & z_D & \sqrt{\gamma^2 - ||\mathbf{z}||^2} \end{bmatrix}, \\
    P_\textrm{norm.}(\mathbf{z}) &= \tfrac{1}{||\mathbf{z}||} \begin{bmatrix} z_1 & \cdots  & z_D \end{bmatrix}, \\
    P_\textrm{homo.}(\mathbf{z}) &= \tfrac{1}{||\mathbf{z}|| + D \gamma} \begin{bmatrix} z_1 & \cdots  & z_D & \gamma_1 & \cdots & \gamma_D \end{bmatrix}, \\
    P_\textrm{(co)sine}(\mathbf{z}) &= \tfrac{1}{||\mathbf{z}||} \begin{bmatrix} z_1 \cos\left(||\mathbf{z}||\right)& \cdots & z_D \cos\left(||\mathbf{z}||\right) & z_1 \sin\left(||\mathbf{z}||\right)& \cdots & z_D \sin\left(||\mathbf{z}||\right) \end{bmatrix},
\end{align*}
where $\gamma = \gamma_1 = \ldots = \gamma_D$ represents the largest norm possible for $\mathbf{z}$ (given that $\mathbf{z}$ is a scaled version of $\mathbf{x} \in [-1,1]^D$).

\begin{figure*}[h!]
    \centering
    \includegraphics[width=\textwidth]{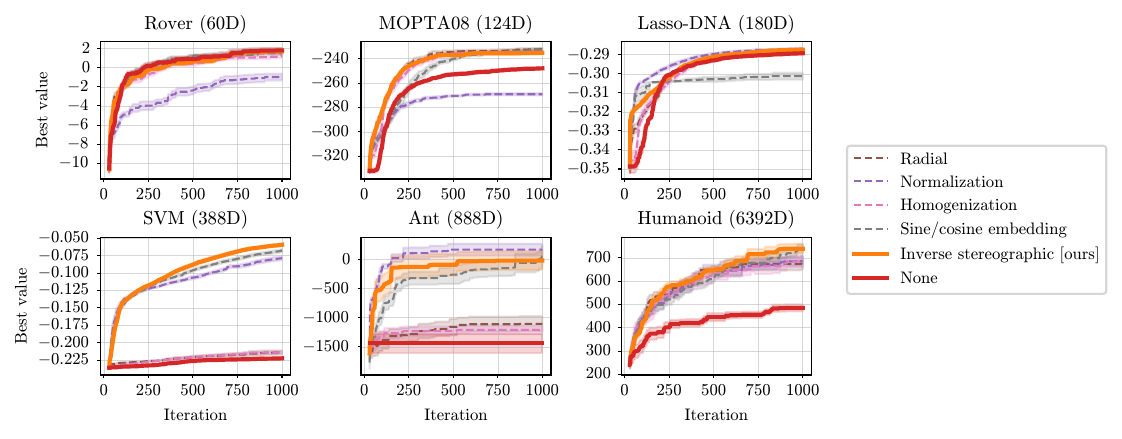}
    \caption{Extension of \Cref{fig:spher-proj-small} (the choice of spherical projection can have a drastic impact on performance).}
    \label{fig:spher-proj-full}
\end{figure*}

This broader comparison of datasets yields consistent conclusions with those presented in \Cref{sec:ablations}: a spherical projection is necessary to improve performance, but the choice of spherical projection can have a drastic impact on performance. For example, simple \texttt{Normalization} achieves state-of-the-art performance on the \texttt{Ant} dataset, but is outperformed by all other projections---including no projection at all---on the \texttt{MOPTA08} dataset. For the final version of our linear kernel, we opt for the \texttt{Inverse}~\texttt{stereographic} projection, which consistently delivers superior performance across datasets.

\subsection{Percentage of Dimensions on the Boundary} \label{app:boundary-plot}
In \Cref{fig:boundary-full}, we complement the results of \Cref{fig:otsd-boundary} by adding four more datasets (for a total of six). As described in \Cref{sec:analysis}, we can clearly observe a corner-seeking behavior for the standard linear kernel, with almost $100\%$ of the dimensions (of acquired points) lying on the boundary. For our linear kernel and Vanilla BO, this is not the case, but we note a surprisingly high diversity of acquired points across datasets: from almost $0\%$ dimensions on the boundary for \texttt{SVM} and \texttt{Ant} to more than $50\%$ for the other four datasets.

\begin{figure*}[h!]
    \centering
    \includegraphics[width=\textwidth]{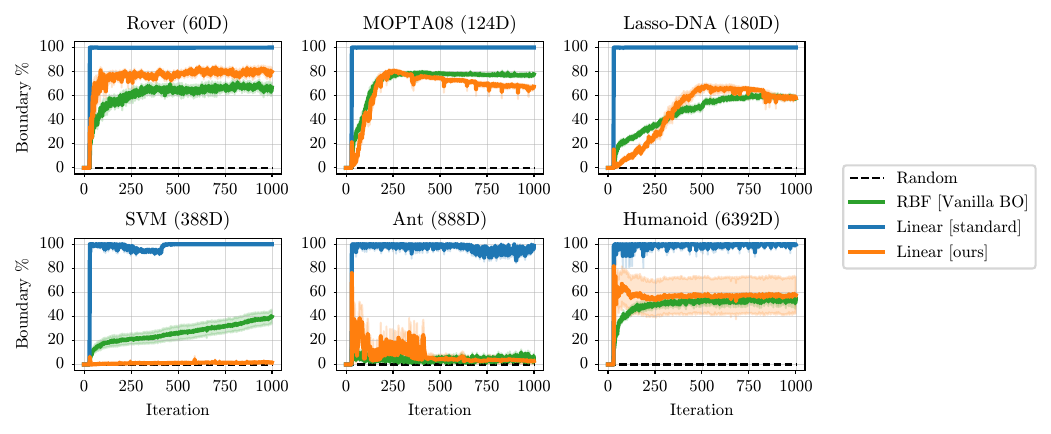}
    \caption{Extension of \Cref{fig:otsd-boundary} (standard linear kernels lead to acquisitions in the corners of the hypercube, whereas our spherically-mapped linear kernels acquire points similar to Vanilla BO).}
    \label{fig:boundary-full}
\end{figure*}

\subsection{Observation Traveling Salesman Distance (OTSD)} \label{app:otsd}

Similar to the above paragraph, we expand the results of \Cref{fig:otsd-boundary} in \Cref{fig:otsd-full} to include four more datasets. Here, we can observe a similar pattern as described in \Cref{sec:analysis}: the corner-seeking behavior of the standard linear kernel leads to an OTSD metric often exceeding that of random search, whereas our spherically-mapped linear kernel displays patterns of \say{locality} closer to those of Vanilla BO. Additionally, we can see that the OTSD metric for the standard linear kernel keeps increasing with dimensionality $D$. This result is not surprising: as $D$ grows large, corners in a high-dimensional hypercube are increasingly far away from each other compared to randomly-sampled points.

\begin{figure*}[h!]
    \centering
    \includegraphics[width=\textwidth]{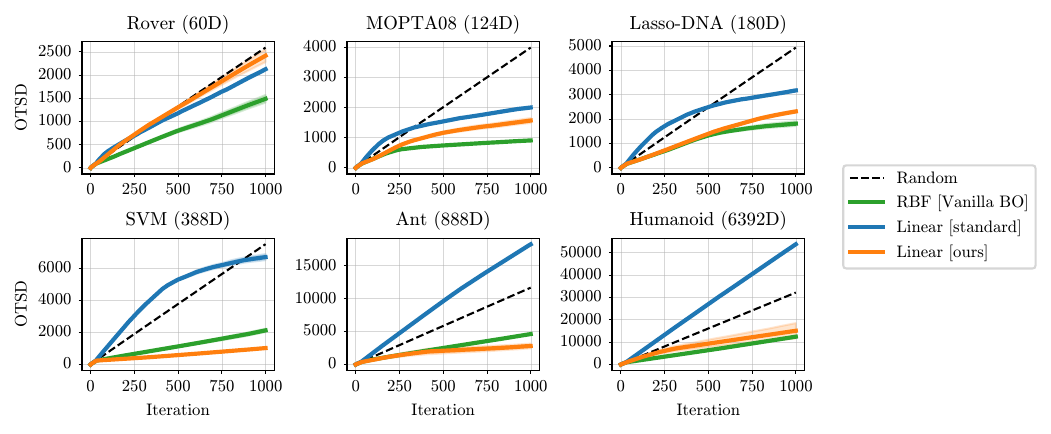}
    \caption{Extension of \Cref{fig:otsd-boundary} (our linear kernel exhibits ``locality'', similar to Vanilla BO and unlike the standard linear kernel).}
    \label{fig:otsd-full}
\end{figure*}

\subsection{Regression Performance} \label{app:boxplot}

In \Cref{fig:boxplot-full}, we add four more datasets (for a total of six) to the Box plots of \Cref{fig:boxplot}. Based on these six datasets, we draw the same conclusions as in \Cref{sec:analysis}: \vspace{-5pt}
\begin{enumerate}
\itemsep0em
    \item the points chosen during the BO loop are easier to predict than evenly-spread Sobol points (perhaps because of the \say{local} behavior of our model), and
    \item our spherically-mapped linear model does not significantly improve the regression performance of the standard linear model (thus discarding the idea that the performance of our linear model is due to strong, non-linear features induced by the spherical projection). \end{enumerate}

\begin{figure*}[h!]
    \centering
    \includegraphics[width=\textwidth]{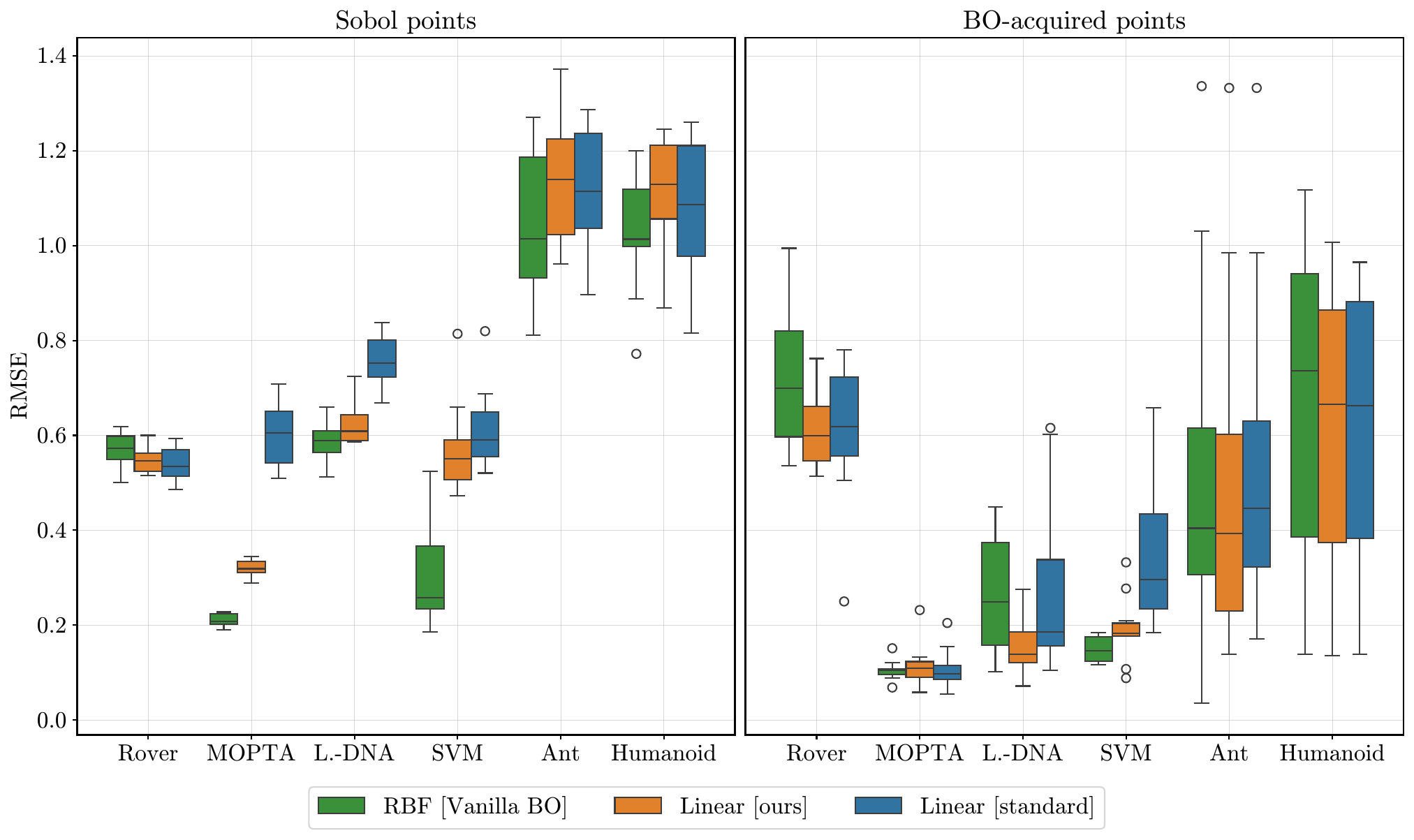}
    \caption{Extension of \Cref{fig:boxplot} (our spherical projection does not lead to better predictive performance for the standard linear kernel).}
    \label{fig:boxplot-full}
\end{figure*}

\end{document}